\documentclass[letterpaper]{article} 
\usepackage[preprint]{aaai2027}  
\usepackage[hyphens]{url}  
\usepackage{graphicx} 
\usepackage{natbib}  
\usepackage{caption} 
\usepackage{hyperref}
\usepackage{algorithm}
\usepackage{algorithmic}
\usepackage[table]{xcolor}
\usepackage{newfloat}
\usepackage{listings}
\DeclareCaptionStyle{ruled}{labelfont=normalfont,labelsep=colon,strut=off} 
\floatstyle{ruled}
\newfloat{listing}{tb}{lst}{}
\floatname{listing}{Listing}

\usepackage{booktabs}
\usepackage{amsmath}
\usepackage{amssymb}
\usepackage{amsthm}
\usepackage{graphicx}
\usepackage{multirow}
\usepackage{xcolor}
\DeclareMathOperator*{\argmin}{arg\,min}
\newtheorem{proposition}{Proposition}

\definecolor{financecolor}{gray}{0.93}  
\definecolor{miracolor}{gray}{1.0}     
\title{Equitable System-Prompt Selection via Constrained Mixed-Strategy GroupDRO}
\author{
    Mengyu Xu\textsuperscript{\rm 1},
    Qiaoxin Yang\textsuperscript{\rm 2},
    Zhihan Liu\textsuperscript{\rm 3},
    Ruiyao Xu\textsuperscript{\rm 3},\\
    Zachary Liu\textsuperscript{\rm 4},
    Kezhen Chen\textsuperscript{\rm 5},
    Chongyang Gao\textsuperscript{\rm 3}\corresponding
}
\affiliations{
\textsuperscript{\rm 1}The University of Chicago
\textsuperscript{\rm 2}SynAI Technologies, Inc.
\textsuperscript{\rm 3}Northwestern University
\textsuperscript{\rm 4}Dartmouth College
\textsuperscript{\rm 5}Analogy AI, Inc.\\
mxu09@uchicago.edu, yangqiaoxinn@qq.com, zhliu0627@gmail.com, ruiyaoxu2028@u.northwestern.edu, ruibo.liu.gr@dartmouth.edu, kezhenchen@analogyai.org, chongyanggao2026@u.northwestern.edu
}

\begin{document}

\maketitle
\pagestyle{plain}   
\thispagestyle{plain}

\begin{abstract}
Large language models are increasingly used for information seeking, yet semantically equivalent questions phrased in different ways can receive answers of considerably different quality. System prompts are widely employed to steer response behavior, but they are typically optimized for average-case quality, so some question phrasings may still receive incomplete or low-quality answers. To address this, we formulate a constrained mixed-strategy GroupDRO framework for system-prompt selection. Instead of optimizing the system-prompt text, the framework assigns weights to system prompts in an existing pool to minimize the worst-case information-quality loss across evaluation metrics and groups, while constraining the mean loss to stay close to that of average-based selection. Because pool generation and selection are decoupled, the method applies to any system-prompt pool and can leverage an ensemble of complementary system prompts rather than a single one. Across five LLMs on two bilingual medical and consumer-finance benchmarks, the constrained method reduces the Overall Mean, Worst 25\% Mean, and Worst by 13.1\%, 13.2\%, and 13.7\% on average relative to no mitigation while keeping overall quality close to Average selection. Its multi-prompt weights reveal complementarity across metric-group pairs. Code and data are available at
\url{https://github.com/Rainxu09/equitable-system-prompt-selection}.
\end{abstract}

\section{Introduction}
Large language models (LLMs) are increasingly used for everyday information seeking in high-stakes domains such as health and personal finance. In these settings, the same underlying question can be phrased in many ways. Ideally, semantically equivalent questions should receive answers of comparable quality. In practice, however, response quality can vary considerably with phrasing, and prior work shows that LLMs can respond differently to content-equivalent inputs depending on language, domain literacy, resource constraints, dialect, persona, or framing signals~\citep{xu2026,hofmann2024,gupta2024}. This disparity has a fairness dimension: users whose phrasing signals low domain literacy are precisely those least able to recover missing or diluted information on their own, and low health literacy is linked to poorer health outcomes~\citep{berkman2011}.

A common mitigation is to add a system prompt that steers response behavior, for example by instructing the model to preserve key information or maintain actionable detail. However, whether written by hand or generated by automated prompt optimizers~\citep{zhou2023,yang2024large,guo2024evo,agrawal2026gepa}, system prompts are typically selected for their average response quality. A system prompt that performs well on average can still leave particular question phrasings with incomplete or low-quality answers, and average-based selection provides no mechanism to detect or correct this.

We address this problem at the selection stage. Strong candidate system prompts are already abundant, coming from human design, LLM generation, or automatic prompt optimizers, so the practical question is not how to write another prompt but how to use the existing ones well. Furthermore, no single system prompt may serve all question phrasings well, so a method that combines complementary prompts offers protection that any single rewritten prompt cannot.
We formulate the problem in a practical regime: given a finite pool of candidate system prompts, predefined evaluation groups, and offline scores for each candidate, we ask how to select from the pool so that the weakest groups also receive high-quality answers. These conditions make the problem exactly solvable rather than a learning task. With a finite pool and fully observed scores, there is no unseen input space to generalize over; a fitted selector would only add estimation error, which could itself fall unevenly across groups.

Building on these considerations, we propose a constrained mixed-strategy GroupDRO method for system-prompt selection, which carries the worst-group principle of group distributionally robust optimization~\citep{sagawa2020dro} and minimax fairness~\citep{hashimoto2018,diana2021minimax} from model training to prompt selection. The selector assigns weights to system prompts so as to minimize the worst-case information-quality loss across evaluation metrics and groups, while a mean-loss constraint keeps average quality close to that of average-based selection. Because the objective is linear in the weights, the problem reduces to a linear program: model agnostic, free of additional training, and applicable to any prompt pool. 
Beyond selection, the optimal weights serve as a diagnostic. If the solution concentrates on a single system prompt, that prompt alone is sufficient under the worst-case objective; if it spreads weight across several, the pool contains complementary prompts that protect different evaluation metrics and groups.

We evaluate this approach across five LLMs on two bilingual information-seeking benchmarks: MIRA, an existing medical information benchmark~\citep{xu2026}, and a new controlled consumer finance benchmark that we construct. In our experiments, the constrained method improves the overall mean score, the mean score of the worst 25\% of metric--group pairs, and the worst-case score relative to the no-mitigation baseline, reducing them by 13.1\%, 13.2\%, and 13.7\% on average, respectively. It achieves consistent gains over single-prompt selectors on the weakest quartile in all 10 model--domain settings, while keeping overall average quality nearly unchanged. The selected weights consistently place mass on multiple system prompts, showing that different system prompts protect different evaluation metrics and groups. Our contributions can be summarized as follows.

\begin{itemize}
\item An equitable system-prompt selection framework with a constrained mixed-strategy GroupDRO selector. We formulate mitigation prompting as selection from an existing candidate pool rather than optimization of prompt text: the selector assigns weights to system prompts to minimize the worst-case information-quality loss across evaluation metrics and groups, subject to a mean-loss constraint that preserves average quality. The method is model agnostic, requires no additional training, and applies to pools from any source. The optimal weights also serve as a diagnostic of prompt complementarity in the pool.

\item A bilingual consumer-finance benchmark. We construct a controlled bilingual consumer-finance benchmark with 60 low-risk seed questions across six categories and fixed information checklists across question phrasings. Finance experts reviewed the seeds, checklists, and rubrics, and we audit agreement between a trained finance annotator and the LLM judge on 250 responses.
\item A cross-domain empirical study. Across bilingual medical and consumer-finance benchmarks and five LLMs, we show that constrained mixed-strategy GroupDRO system-prompt selection improves answer quality for the weakest metric--group pairs compared with the no-mitigation baseline and single-prompt selectors, while keeping average quality nearly unchanged.
\end{itemize}

\section{Background and Related Work}
\paragraph{Group-robust optimization and minimax fairness.}
GroupDRO and minimax fairness focus on the worst-performing group, not only the average case~\citep{sagawa2020dro,hashimoto2018,diana2021minimax}. Most prior work uses this idea during training or fine-tuning, where the method updates model weights and evaluates groups defined in the training data. We apply the same worst-group idea to system-prompt selection. Without changing the model or rewriting the system prompt, we select from an existing pool to reduce worst-group information-quality loss.

\paragraph{Prompt optimization and system-prompt selection.}
A broader review of prompt optimization and system-prompt selection
is provided in Appendix~\ref{app:relatedwork}. Unlike prompt
optimization methods that generate or modify prompt text, our method
operates on an existing system-prompt pool. Among existing
selection approaches, the closest to our setting is Prompt Risk
Control (PRC)~\citep{zollo2024}, which uses calibration data to select a prompt whose risk satisfies a chosen bound with high probability. We include PRC as a baseline in our experiments. 

Prompt optimization and selection methods generally seek a single
prompt that performs well under a distribution-level criterion, such
as average quality or a tail-risk bound. Such criteria do not identify or optimize the weakest evaluation groups: a prompt that is strong on average, or whose response-level tail is controlled, can still leave particular groups with weaker answers. Given an existing system-prompt pool, we instead select a system prompt, or a weighted mixture, that protects the weakest metric--group pairs while preserving average quality. To our knowledge, this is the first work to formulate system-prompt selection as a constrained group-robust optimization problem.

\paragraph{LLM fairness and question phrasing effects.}
Our work also builds on fairness and bias evaluation in LLMs. Prior work shows that LLMs can respond differently to content-equivalent inputs depending on language, literacy, dialect, persona, or framing signals~\citep{xu2026,hofmann2024,gupta2024}. In medical and financial information seeking, low domain literacy is especially important because users may be less able to fill in missing or diluted information. Prior work also links low health literacy to poorer health outcomes~\citep{berkman2011}. We do not treat all groups as fairness groups. Low health- and financial-literacy signals give the clearest fairness interpretation, while language, register, query skeleton, resource constraints, and expression framing mainly test robustness to different ways of asking the same question.

\section{Methods}
\subsection{Problem formulation}
\paragraph{Setup and notation.}
Let $\mathcal{P}=\{p_1,\dots,p_N\}$ denote a fixed pool of $N$ candidate system prompts. 
For an application domain, let $\mathcal{G}$ denote its predefined evaluation groups. The construction of $\mathcal{G}$ is domain-specific and is described in Section~\ref{sec:dataset}. Each group corresponds to a distinct profile of question characteristics applied to the same underlying question. Depending on the domain, these characteristics include language, literacy signal, resource constraint, register, question skeleton, and framing.
Let $\mathcal{M}$ denote the set of evaluation metrics. In our experiments, $\mathcal{M}=\{m_1, m_2, m_3\}$, where $m_1$ measures information dilution, $m_2$ measures completeness, and $m_3$ measures actionability. Information dilution~\citep{xu2026} refers to responses that address the question but omit or weaken substantive information, such as underlying mechanisms, relevant thresholds, or risk boundaries.

For system prompt $p\in\mathcal{P}$, evaluation group $g\in\mathcal{G}$, and evaluation metric $m\in\mathcal{M}$, let $L^{m}_{p,g}$ denote the empirical information-quality loss of system prompt $p$ on group $g$ for metric $m$. Lower values indicate better responses. We assume that all losses are nonnegative.
For each application domain, we construct a joint loss matrix. Rows correspond to candidate system prompts, columns correspond to metric–group pairs in $\mathcal{M}\times\mathcal{G}$, and each entry is the mean LLM-judge score of responses generated for system prompt $p$, group $g$, and metric $m$.
Because lower scores indicate better responses, our optimization treats these scores as losses to minimize. The equations below optimize jointly over evaluation metrics and groups.

\paragraph{Single-Prompt Selection.}
A common approach is to select one system prompt from the fixed pool and use it across all evaluation groups.
Average-loss selection selects one system prompt by minimizing mean loss across evaluation metrics and groups:
\begin{equation}
p_{\mathrm{avg}} = \argmin_{p\in\mathcal{P}}\frac{1}{|\mathcal{M}||\mathcal{G}|}\sum_{m\in\mathcal{M}}\sum_{g\in\mathcal{G}} L^{m}_{p,g}.
\label{eq:avg}
\end{equation}

\noindent However, a system prompt that performs well on average may still underperform on certain metric-group pairs. Pure GroupDRO~\cite{sagawa2020dro} addresses this by selecting the system prompt with the lowest worst-case loss:
\begin{equation}
p^\star_{\mathrm{pure}} = \argmin_{p\in\mathcal{P}}\max_{\substack{m\in\mathcal{M}\\g\in\mathcal{G}}}L^{m}_{p,g}.
\label{eq:pure}
\end{equation}

\noindent Both methods rely on a single system prompt, which can be limiting when different system prompts help different metric-group pairs. Moreover, a single prompt may not balance the trade-off between average quality and worst-case protection: although Pure GroupDRO finds the best single prompt under the joint worst-case objective, that prompt may still fail to provide comparable answer quality across all evaluation metrics and groups. These limitations motivate the constrained mixed-strategy formulation introduced next.

\subsection{Constrained Mixed-strategy GroupDRO}
Mixed strategies are a standard tool in minimax problems~\citep{vonneumann1928} and robust decision-making~\citep{sessa2020mixed}. We use this framework as a relaxation of single-system-prompt selection. Mixed-strategy GroupDRO uses the same fixed pool but assigns weights to system prompts, testing whether this weighted mixture can achieve lower worst-case loss across metric–group pairs than the best single system prompt. 
To prevent this worst-case gain from degrading average quality, we further constrain the mean loss of the mixture to stay close to that of Average selection. If several system prompts receive weight under this constraint, this indicates that different system prompts in the pool help different evaluation metrics or groups. In other words, the pool contains complementary system prompts.

\paragraph{Weighted system-prompt mixture.}
Mixed GroupDRO chooses a distribution over system prompts, which we interpret as weights over the system-prompt pool. Let $w = (w_1,\dots,w_N)\in\Delta_N$ denote these weights, where $\Delta_N = \{w \in \mathbb{R}^N_{\ge 0} : \sum_{i=1}^{N} w_i = 1\}$ and $w_i$ is the weight assigned to system prompt $p_i$. The expected loss for metric $m$ and evaluation group $g$ under weights $w$ is:
\begin{equation}
R_{m,g}(w)=\sum_{i=1}^{N}w_i L^{m}_{p_i,g}.
\label{eq:Rg}
\end{equation}

\noindent Mixed GroupDRO chooses the system-prompt weights that minimize the worst-case expected loss.

\paragraph{Mean-constrained mixture.}
To keep average quality comparable to Average selection, we constrain the mean loss of the weighted mixture. Let
\begin{equation}    
\bar{L}(w)=\frac{1}{|\mathcal{M}||\mathcal{G}|}\sum_{m\in\mathcal{M}}\sum_{g\in\mathcal{G}}R_{m,g}(w)
\end{equation}
denote the mean loss under weights $w$. Constrained Mixed GroupDRO chooses the system-prompt weights that minimize the worst-case expected loss while limiting the mean loss:
\begin{equation}
\begin{aligned}
w^\star_{\epsilon} = \arg\min_{w\in\Delta_N} \quad & \max_{\substack{m\in\mathcal{M}\\ g\in\mathcal{G}}} R_{m,g}(w) \\
\text{subject to} \quad & \bar{L}(w) \leq (1+\epsilon)\,\bar{L}_{\mathrm{avg}},
\label{eq:constrained}
\end{aligned}
\end{equation}
where\begin{equation}
\bar{L}_{\mathrm{avg}} = \frac{1}{|\mathcal{M}||\mathcal{G}|}\sum_{m\in\mathcal{M}}\sum_{g\in\mathcal{G}}L^{m}_{p_{\mathrm{avg}},g}
\end{equation}
denotes the mean loss achieved by $p_{\mathrm{avg}}$. We assume $\epsilon\ge 0$. In our main experiments, we set $\epsilon=0.005$, so the development mean loss can be at most 0.5\% higher than that of Average selection.

For later use, define the constrained mixed feasible set as
$
\Delta_{N,\epsilon} = \left\{w\in\Delta_N: \bar{L}(w) \le (1+\epsilon)\bar{L}_{\mathrm{avg}}\right\}.
$

\paragraph{Linear program formulation.}
The Constrained Mixed GroupDRO problem in Eq.~\eqref{eq:constrained} minimizes the maximum over finitely many metric–group pairs, and both $R_{m,g}(w)$ and $\bar L(w)$ are linear in $w$. It can therefore be written as a linear program. We introduce a variable $t$ that upper-bounds the expected loss of every metric–group pair:
\begin{equation}
\begin{aligned}
\min_{w,t}\quad & t \\
\text{s.t.}\quad
& R_{m,g}(w)\le t,
    \quad \forall m\in\mathcal{M},\ \forall g\in\mathcal{G},\\
& \bar L(w)\le (1+\epsilon)\bar L_{\mathrm{avg}},\\
& \sum_{i=1}^{N} w_i=1,\\
& w_i\ge 0,
    \quad i = 1, \dots, N.
\end{aligned}
\label{eq:lp}
\end{equation}
The second constraint limits the mean loss relative to average selection. Removing this constraint gives the unconstrained Mixed GroupDRO formulation.

For a fixed feasible $w$, the smallest feasible value of $t$ is $\max_{\substack{m\in\mathcal{M}\\g\in\mathcal{G}}}R_{m,g}(w),$ so the linear program in Eq.~\eqref{eq:lp} is equivalent to the Constrained Mixed GroupDRO objective in Eq.~\eqref{eq:constrained}. In other words, minimizing $t$ minimizes the worst-case expected loss while satisfying the mean-loss constraint.

To compare single-system-prompt and mixed selection under the same constraint, let $e_i\in\Delta_N$ denote the one-hot vector that places all weight on system prompt $p_i$. Define the feasible set of single system prompts as
\[
\mathcal{P}_{\epsilon} = \left\{p_i\in\mathcal{P}: \bar L(e_i) \le (1+\epsilon)\bar L_{\mathrm{avg}}\right\}.
\]
Since $R_{m,g}(e_i)=L^m_{p_i,g}$, $\bar L(e_i)$ is exactly the mean loss of system prompt $p_i$. For $\epsilon\ge0$, this set is nonempty because the one-hot vector corresponding to $p_{\mathrm{avg}}$ has mean loss  $\bar L_{\mathrm{avg}}\le(1+\epsilon) \bar L_{\mathrm{avg}}$, and therefore $p_{\mathrm{avg}}\in\mathcal P_\epsilon$.

The constrained pure and mixed worst-case values are
\[
V_{\mathrm{pure},\epsilon} = \min_{p_i\in\mathcal{P}_\epsilon}\max_{\substack{m\in\mathcal{M}\\g\in\mathcal{G}}}R_{m,g}(e_i)
\]
and, using the constrained feasible set $\Delta_{N,\epsilon}$,
\[
V_{\mathrm{mix},\epsilon} = \min_{\substack{w\in\Delta_N\\
\bar L(w)\le(1+\epsilon)\bar L_{\mathrm{avg}}}}
\max_{\substack{m\in\mathcal{M}\\g\in\mathcal{G}}} R_{m,g}(w).
\]
Here, $V_{\mathrm{pure},\epsilon}$ is the best worst-case value among single system prompts satisfying the mean-loss constraint. It is distinct from the Pure GroupDRO objective in Eq.~\eqref{eq:pure}.

\begin{proposition}[Constrained Mixed GroupDRO dominance]
\label{prop:collapse}
For \(\epsilon\ge 0\), $V_{\mathrm{mix},\epsilon} \leq V_{\mathrm{pure},\epsilon}$. Equality holds if and only if there is an optimal constrained mixed solution that places all weight on a single prompt in $\mathcal P_\epsilon$. When $V_{\mathrm{mix},\epsilon} < V_{\mathrm{pure},\epsilon}$, every optimal constrained mixed solution places positive weight on at least two system prompts.
\end{proposition}

The proof follows because every mean-feasible single system prompt corresponds to a one-hot vector in the constrained mixed feasible set $\Delta_{N,\epsilon}$. Thus, feasible single-prompt choices form a subset of the constrained mixed feasible set, and the mixed optimum cannot be worse than the best feasible single prompt. Equality holds when a one-hot vector attains the mixed optimum. A full proof is provided in Appendix~\ref{app:proof}.

The inequality is a development-matrix guarantee, not a test guarantee. We therefore use Constrained Mixed GroupDRO to diagnose system-prompt-pool complementarity. If an optimal constrained mixed solution places all weight on one system prompt, the best mean-feasible single system prompt is already minimax-optimal under the constrained objective. If $V_{\mathrm{mix},\epsilon} < V_{\mathrm{pure},\epsilon}$, no mean-feasible single system prompt achieves the same worst-case development loss, providing evidence of system-prompt complementarity.

\section{Experiments}
\subsection{System-prompt Pool Construction}
Our method takes the candidate system-prompt pool $\mathcal{P}$ as input rather than optimizing system-prompt text directly. It can be constructed from sources, including human-written system prompts, LLM-generated system prompts, prompt-optimization methods, or combinations of these sources.

In our experiments, we generate $\mathcal{P}$ with GPT-5.4~\citep{openai2026gpt54} using human-specified mitigation goals, then manually review the system prompts before selection. Each domain uses a fixed pool of 55 system prompts, organized into 11 mitigation families with 5 wording variants per family. The families target broad mitigation goals such as preserving checklist information, avoiding over-referral, maintaining actionable detail, and adapting explanations without diluting domain content. Within each domain, the same system-prompt pool is used across all evaluation metrics and groups for all four selectors: Average selection, Pure GroupDRO, Mixed GroupDRO, and Constrained Mixed GroupDRO. This lets us test whether selecting from a general system-prompt pool can improve worst-case robustness without writing separate system prompts for different evaluation metrics or groups.

\subsection{Benchmarks and Evaluation Groups}
\label{sec:dataset}
We evaluate on MIRA, an existing bilingual medical-information benchmark~\citep{xu2026}, and on a new consumer-finance benchmark constructed with the same controlled evaluation design. Our finance benchmark contains 60 low-risk consumer-finance seed questions, balanced across six categories: securities, trusts, insurance, consumer protection, credit and credit reporting, and deposits and payments, with 10 seeds per category. For each seed, we define a checklist of the key information that a useful answer should include, and use the same checklist across all phrasings of that question. We create 24 finance groups by crossing two languages (English and Chinese), two financial-literacy signals (high and low), two resource constraints (high and low), and three expression frames (direct, confusion, and misconception), yielding 1,440 finance question variants. Resource constraint indicates whether the user has high or low assets. Expression frame shows direct, confused, or misconception-based phrasing.

The consumer-finance responses were annotated by a trained finance expert. Before annotation, the finance seeds, checklists, and scoring rubrics were additionally reviewed by domain experts, including a professor in finance and economics. Appendix~\ref{app:finance-agreement} reports agreement between the finance annotator and the LLM judge on 250 audited responses. Exact agreement is the proportion of identical ratings, adjacent agreement is the proportion of ratings that differ by at most one point, and QWK denotes quadratic weighted kappa. Agreement is consistently high across the three evaluation metrics, with exact agreement ranging from 0.804 to 0.924, adjacent agreement from 0.964 to 0.996, and QWK from 0.832 to 0.923. Each domain contains 60 underlying questions, with 20 used for development and 40 held out for testing. The split preserves coverage across categories, language, literacy signal, and question framing. Responses are scored on three information-quality metrics: $m_1$ (information dilution), $m_2$ (completeness), and $m_3$ (actionability), with lower scores indicating better responses. All three metrics are used jointly for system-prompt selection and evaluation.

\subsection{Experimental Protocol}
\label{sec:pro}
We evaluate five LLMs: Qwen36Plus~\citep{yang_qwen3_2025}, DeepSeek-V4Pro~\citep{deepseek2026deepseekv4}, GLM5~\citep{glm5team2026}, Gemma4-31B~\citep{gemmateam2026}, and Llama4Scout~\citep{llama4}. For each domain, we evaluate 60 seed questions across 24 evaluation groups, yielding 1,440 question variants in total. Each model is run on the 55-system-prompt pool and on a no-mitigation baseline, producing 792,000 system-prompt-pool responses and 14,400 baseline responses across the two domains, for 806,400 model responses. Responses are scored by GPT-5.4-mini~\citep{gpt54mini} as the LLM judge~\citep{zheng_judging_2023} using the MIRA scoring framework, adapted to consumer-finance content while keeping the same score direction and metrics. Baseline responses are scored in the same way but are not included in system-prompt selection. They are used only to measure information-quality differences before mitigation.

For each model and domain, we score every system prompt on all metric–group pairs in the development set. This gives a joint development score matrix with system prompts as rows and the 72 metric–group pairs as columns. We then apply four selectors to the same matrix. Average selection minimizes the mean score across all pairs, Pure GroupDRO minimizes the worst-case score, Mixed GroupDRO minimizes the worst-case score using system-prompt weights, and Constrained Mixed GroupDRO solves the LP in \eqref{eq:lp} with the mean-loss constraint. For Constrained Mixed GroupDRO, we compare $\epsilon\in\{0,0.005,0.01,0.02\}$ using the development-set tradeoff between the Overall mean and the Worst 25\% mean. We use $\epsilon=0.005$ in the main experiments because it lies on the development Pareto frontier in both domains. 

As an additional single-system-prompt baseline, we evaluate Prompt Risk Control (PRC)~\citep{zollo2024} using response-level development losses. For each candidate system prompt, we construct a Bonferroni-corrected one-sided KS upper bound on CVaR at 25\% and select the prompt with the smallest bound. Because the controlled question variants are not fully independent, we use PRC as a comparative selector rather than claiming nominal coverage. We evaluate the system prompts selected by Average selection, Pure GroupDRO, and PRC, along with the weights selected by Constrained Mixed GroupDRO and the unconstrained Mixed GroupDRO comparator, on held-out test questions.

We report three evaluation statistics. The Overall Mean is averaged over all 72 metric-group pairs, the Worst 25\% Mean is averaged over the 18 highest-loss pairs, and the worst-case score is the maximum score among all 72 pairs. We also compare Constrained Mixed GroupDRO with Average selection across all model-question-variant-metric scores in the full evaluation set. This analysis reports how many scores improve, worsen, or remain unchanged, the mean change among the improved and worsened scores, and the overall net change.

For the candidate-growth analysis, we sample 500 system-prompt subsets of each size $K\in\{5,15,25,35,45\}$ from the same fixed pool of 55 system prompts and refit the selectors on each subset. For $K=55$, we use the full pool once. The development and test questions remain unchanged. We then track how the Overall Mean, Worst 25\% Mean, and worst-case score change as $K$ grows.

\subsection{Validating the Evaluation Groups}
\label{sec:validity}

A worst-group objective is meaningful only if the groups reflect real differences in response quality. The grouping factors are task-motivated: language, domain-literacy signal, and question framing describe how the same question is phrased. They are not result-driven groups. Empirically, both consumer-finance and MIRA baseline responses show clear gaps between average and worst-group scores. The gaps are positive across all model-domain-metric settings, with detailed results reported in Appendix~\ref{app:gaps}. We therefore use the 24 groups to measure robustness across different ways of phrasing the same question.

\begin{table}[t]
\centering
{\small
\setlength{\tabcolsep}{2.2pt}
\begin{tabular}{@{}lccc@{}}
\toprule
\textbf{Predictor}
& $\mathbf{m_1}$ & $\mathbf{m_2}$ & $\mathbf{m_3}$ \\
\midrule
\multicolumn{4}{@{}l}{\textit{MIRA}} \\
Chinese (vs. Eng.) & $-0.096^{*}$ & $-0.055$ & $-0.049$ \\
Low HLS (vs. High) & $0.118^{***}$ & $0.114^{***}$ & $0.026$ \\
Colloquial (vs. Formal) & $0.022$ & $0.010$ & $-0.007$ \\
Skeleton S2 (vs. S1) & $0.011$ & $0.001$ & $0.006$ \\
Skeleton S3 (vs. S1) & $0.028^{*}$ & $0.014$ & $0.012$ \\
\midrule
\multicolumn{4}{@{}l}{\textit{Consumer finance}} \\
Chinese (vs. Eng.) & $0.006$ & $0.025$ & $0.003$ \\
Low FLS (vs. High) & $0.094^{***}$ & $0.094^{***}$ & $0.060^{***}$ \\
Low asset (vs. High) & $0.017^{*}$ & $0.019^{**}$ & $-0.010$ \\
Confusion (vs. Direct) & $0.045^{***}$ & $0.045^{***}$ & $0.041^{***}$ \\
Misconception (vs. Direct) & $0.001$ & $-0.028$ & $-0.024$ \\
\midrule
\multicolumn{4}{@{}l@{}}{%
\textit{Note.} Positive coefficients indicate worse outcomes.}\\
\multicolumn{4}{@{}l@{}}{%
$^{*}p<.05$, $^{**}p<.01$, and $^{***}p<.001$.}\\
\multicolumn{4}{@{}l@{}}{%
Full model-specific results appear in Appendix~\ref{app:group-validity}.}\\
\bottomrule
\end{tabular}
}
\caption{Controlled score differences pooled across five models. Reference groups are shown in parentheses.}
\label{tab:group-validity-summary}
\end{table}

To validate the evaluation groups, we fit fixed-effects regressions on development system-prompt-pool responses, pooling over five response-generating models and controlling for seed question, system prompt, and model, with standard errors clustered by seed question. Table~\ref{tab:group-validity-summary} reports representative coefficients. The strongest and most consistent effects come from literacy signals, while question structure, language, and asset/register factors also affect scores in domain-specific ways. This supports our use of full evaluation groups for joint worst-case system-prompt selection.

A matched-pair placebo test further supports the literacy results. In both domains, the observed low--high literacy differences exceed all 1000 random label flips ($p_{\mathrm{perm}}=1/1001<0.001$). Detailed results are reported in Appendix~\ref{app:placebo}.

Our groups are based on how the question is phrased, not on inferred user identity. We do not ask the model to infer a person's real education level, income, health status, or financial expertise. Instead, we vary the phrasing of the same information need while keeping the underlying checklist fixed. We use fairness language only for the literacy factor because low health- or financial-literacy signals may come from people who are less able to fill in missing information on their own. Other factors, such as language, register, skeleton, asset level, and question frame, are used to measure robustness to different ways of asking the same question.

\subsection{Results}
\subsubsection{Constrained Mixed GroupDRO Improves Average and Worst-Case Quality}
Tables~\ref{tab:mira} and~\ref{tab:finance} show that Constrained Mixed GroupDRO improves both average and worst-case quality over the no-mitigation baseline for all five models in both domains. All results use a single set of system-prompt weights jointly optimized over the 72 metric–group pairs, with development mean loss constrained to at most 0.5\% above Average selection.

The effect is especially clear in consumer finance, where Constrained Mixed GroupDRO further improves answer quality even when baseline scores are already low. For Qwen36Plus, the mean/worst scores decrease from 1.421/1.750 to 1.266/1.513. For DeepSeek-V4Pro, they decrease from 1.497/1.950 to 1.280/1.458. The same overall pattern holds for the remaining models and in MIRA.

\begin{table}[t]
\centering
{\small
\setlength{\tabcolsep}{0.85mm}
\begin{tabular}{@{}llcccccc@{}}
\toprule
Model & Stat. & Base. & Avg & PRC & GDRO & Mix. & Con. Mix. \\
\midrule
Qwen
& All   & 2.256 & \textbf{1.872} & 1.890 & 1.922 & 1.896 & \underline{1.894} \\
& 25\%  & 2.412 & 2.081 & 2.129 & 2.117 & 2.071 & \textbf{2.077} \\
& Worst & 2.605 & 2.200 & 2.225 & \textbf{2.184} & 2.138 & \textbf{2.184} \\
\midrule

DeepSeek
& All   & 2.191 & 1.927 & \textbf{1.868} & 1.927 & 1.892 & 1.893 \\
& 25\%  & 2.361 & 2.134 & 2.117 & 2.134 & 2.070 & \textbf{2.073} \\
& Worst & 2.500 & 2.211 & 2.250 & 2.211 & 2.186 & \textbf{2.195} \\
\midrule

GLM5
& All   & 2.430 & \textbf{1.989} & 2.019 & 2.005 & 2.006 & \underline{1.993} \\
& 25\%  & 2.574 & 2.203 & 2.250 & 2.242 & 2.226 & \textbf{\underline{2.202}} \\
& Worst & 2.650 & \textbf{2.300} & 2.425 & 2.316 & 2.335 & 2.305 \\
\midrule

Gemma
& All   & 2.539 & \textbf{2.107} & 2.210 & \textbf{2.107} & 2.172 & \underline{2.113} \\
& 25\%  & 2.678 & 2.286 & 2.454 & 2.286 & 2.312 & \textbf{\underline{2.283}} \\
& Worst & 2.800 & 2.375 & 2.525 & 2.375 & 2.373 & \textbf{2.366} \\
\midrule

Llama
& All   & 2.747 & 2.598 & \textbf{2.582} & 2.598 & 2.622 & \underline{2.608} \\
& 25\%  & 2.942 & 2.773 & 2.776 & 2.773 & 2.768 & \textbf{\underline{2.759}} \\
& Worst & 3.125 & 2.900 & 2.900 & 2.900 & 2.878 & \textbf{2.865} \\
\bottomrule
\end{tabular}
}
\caption{Held-out MIRA test results.}
\label{tab:mira}
\end{table}

\begin{table}[t]
\centering
{\small
\setlength{\tabcolsep}{0.85mm}
\begin{tabular}{@{}llcccccc@{}}
\toprule
Model & Stat. & Base. & Avg & PRC & GDRO & Mix. & Con. Mix. \\
\midrule
Qwen
& All   & 1.421 & \textbf{1.264} & \textbf{1.264} & 1.284 & 1.294 & \underline{1.266} \\
& 25\%  & 1.629 & 1.390 & 1.390 & 1.394 & 1.394 &
  \textbf{\underline{1.383}} \\
& Worst & 1.750 & 1.550 & 1.550 & \textbf{1.475} & 1.483 & 1.513 \\
\midrule

DeepSeek
& All   & 1.497 & \textbf{1.260} & \textbf{1.260} & 1.287 & 1.280 & 1.280 \\
& 25\%  & 1.715 & 1.407 & 1.407 & 1.397 & 1.375 & \textbf{1.383} \\
& Worst & 1.950 & \textbf{1.450} & \textbf{1.450} & 1.500 & 1.438 & 1.458 \\
\midrule

GLM5
& All   & 1.677 & 1.481 & 1.481 & 1.482 & 1.509 &
  \textbf{\underline{1.459}} \\
& 25\%  & 1.896 & 1.694 & 1.694 & 1.722 & 1.667 &
  \textbf{\underline{1.651}} \\
& Worst & 2.025 & 1.825 & 1.825 & 1.850 & 1.769 & \textbf{1.787} \\
\midrule

Gemma
& All   & 2.059 & 1.720 & 1.720 & 1.757 & 1.757 &
  \textbf{\underline{1.692}} \\
& 25\%  & 2.369 & 1.987 & 1.987 & 1.999 & 1.959 &
  \textbf{\underline{1.916}} \\
& Worst & 2.425 & 2.175 & 2.175 & 2.225 & 2.156 &
  \textbf{2.043} \\
\midrule

Llama
& All   & 2.415 & 2.299 & 2.299 & 2.299 & 2.306 &
  \textbf{\underline{2.295}} \\
& 25\%  & 2.794 & 2.740 & 2.708 & 2.708 & 2.685 &
  \textbf{\underline{2.683}} \\
& Worst & 3.125 & 3.050 & 2.975 & 2.975 & 2.951 &
  \textbf{2.956} \\
\bottomrule
\end{tabular}
}
\caption{Held-out consumer-finance test results.}
\label{tab:finance}
\end{table}

\subsubsection{Mean-Tail Pareto Tradeoff}
Constrained Mixed GroupDRO introduces a tradeoff between mean quality and protection of poorly performing metric-group pairs. We vary the allowed development mean increase over Average selection across $\{0\%,0.5\%,1\%,2\%\}$ and the unconstrained setting. For each setting, we plot the overall mean score against the mean score of the worst-performing 25\% of metric-group pairs, averaged over five models. Lower values on both axes are better. This shows how much the weakest 25\% improve for a given change in average quality. The optimization itself still minimizes the joint worst-case score. We also include PRC and Pure GroupDRO as single-system-prompt comparators. Because neither comparator varies $\epsilon$, they are shown as separate points.

Figure~\ref{fig:test-pareto} shows the held-out mean-tail tradeoff for consumer finance and MIRA. The corresponding development analysis is reported in Appendix~\ref{app:pareto}. Constrained Mixed (0.5\%) provides a strong balance between overall mean quality and the Worst 25\% mean on both development and held-out data. In consumer finance, it achieves the lowest overall mean and worst-25\% scores among the tested settings. In MIRA, it substantially improves the worst-25\% mean while keeping the overall mean close to Average selection. It also achieves lower overall mean and worst-25\% mean than PRC and Pure GroupDRO in both domains. We therefore use the 0.5\% constraint as the main setting in our experiments.

To examine whether these tail improvements come at the expense of other scores, Table~\ref{tab:overall-score-changes} compares Constrained Mixed GroupDRO with Average selection across the full evaluation set. Although slightly more scores worsen than improve on held-out data, the improvements are larger. The resulting overall net change is $-0.007$ in consumer finance and $+0.002$ in MIRA, showing that the overall mean remains nearly unchanged. Together with the mean-tail tradeoff, this shows that Constrained Mixed GroupDRO improves poorly performing metric-group pairs while keeping overall mean quality close to Average selection.

\begin{figure}[t]
    \centering
    \includegraphics[width=\linewidth]{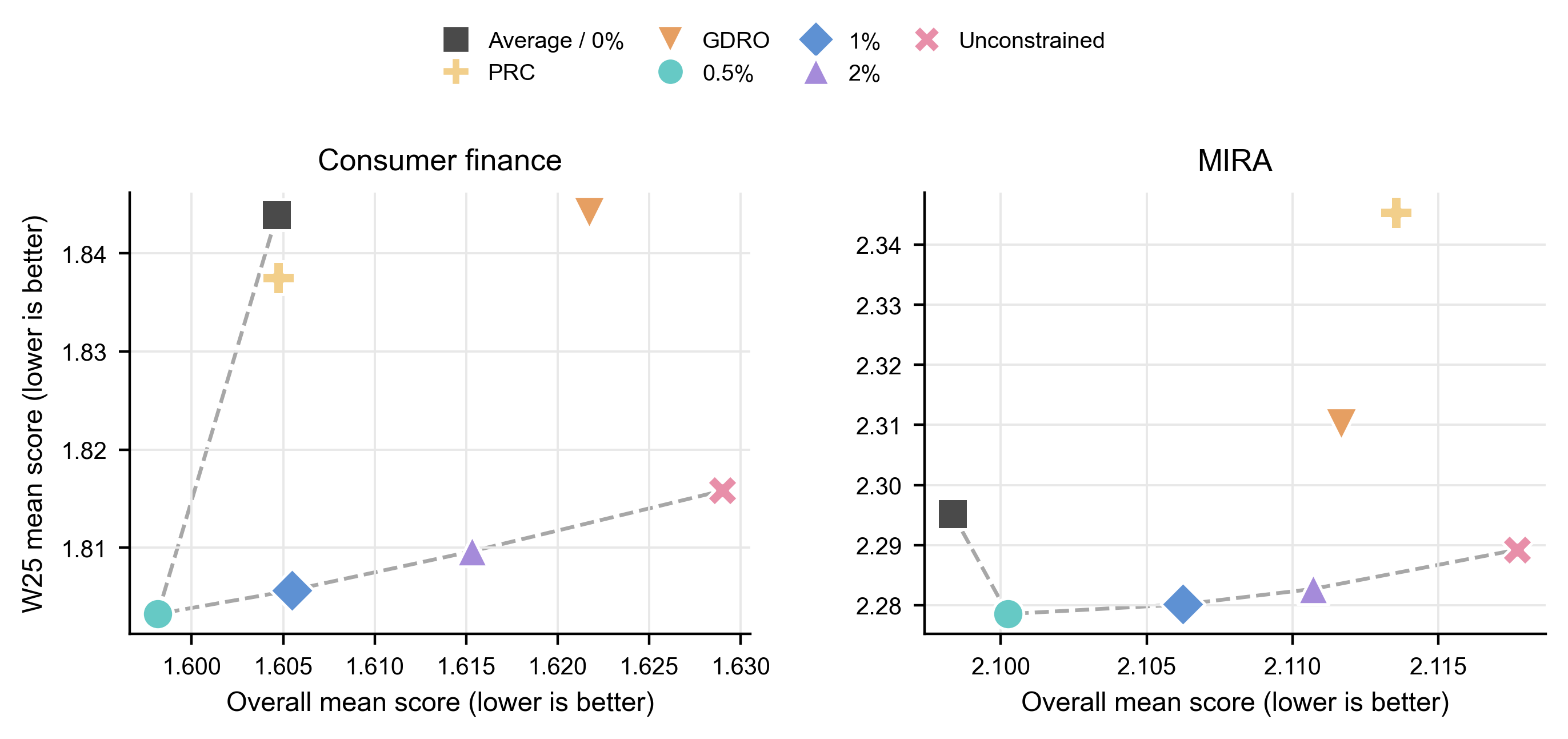}
    \caption{Held-out mean-tail tradeoff for consumer finance and MIRA}
    \label{fig:test-pareto}
\end{figure}

\begin{table}[t]
\centering
{\small
\setlength{\tabcolsep}{0.4mm}
\begin{tabular}{@{}lrrrrrr@{}}
\toprule
& \multicolumn{2}{c}{Improved}
& \multicolumn{2}{c}{Worsened}
& \multicolumn{1}{c}{Unch.} & \\
\cmidrule(lr){2-3}
\cmidrule(lr){4-5}
\cmidrule(lr){6-6}
Test
& \multicolumn{1}{c}{$n$ (\%)}
& \multicolumn{1}{c}{$\bar{\Delta}$}
& \multicolumn{1}{c}{$n$ (\%)}
& \multicolumn{1}{c}{$\bar{\Delta}$}
& \multicolumn{1}{c}{$n$ (\%)}
& \multicolumn{1}{c}{Net $\Delta$} \\
\midrule
Fin.
& 2841 (19.7) & $-0.55$
& 3160 (21.9) & $+0.47$
& 8399 (58.3) & $-0.007$ \\

MIRA
& 3840 (27.1) & $-0.28$
& 4189 (29.6) & $+0.26$
& 6128 (43.3) & $+0.002$ \\
\bottomrule
\end{tabular}
}
\caption{Full evaluation set test score change for Constrained Mixed GroupDRO relative to Average selection. 243 MIRA comparisons with N/A ratings are excluded.}
\label{tab:overall-score-changes}
\end{table}

\subsubsection{Constrained Mixed GroupDRO Reveals System-Prompt Complementarity}
Throughout this subsection, worst-case denotes the maximum score among all 72 metric–group pairs, whereas Worst 25\% mean denotes the average score across the 18 highest-loss pairs.

On development data, Constrained Mixed GroupDRO lowers the joint worst-case score relative to Average selection for all five models in both domains and relative to Pure GroupDRO in 9 of the 10 model-domain settings. The full development results are reported in Appendix~\ref{app:dev-results}. Every constrained mixed solution assigns positive weight to multiple system prompts, with support sizes ranging from 2 to 6 in MIRA and from 3 to 4 in consumer finance. This provides direct evidence of complementarity across metric-group pairs.

In Tables~\ref{tab:mira} and~\ref{tab:finance}, bold marks the best score among Average selection, Pure GroupDRO, PRC, and Constrained Mixed GroupDRO for each statistic, while underlining marks Constrained Mixed overall mean or worst 25\% mean scores that are lower than unconstrained Mixed GroupDRO. On test data, Constrained Mixed GroupDRO lowers the worst-case score in 4/5 consumer-finance settings compared with each single-system-prompt selector. In MIRA, the corresponding counts are 4/5, 4/5 and 5/5. 
PRC remains competitive. On MIRA DeepSeek, it improves the held-out Overall Mean and Worst 25\% Mean relative to Average selection, although its Worst score is higher.

On test data, Constrained Mixed GroupDRO achieves a lower Worst 25\% mean than Average selection, Pure GroupDRO, and PRC in all 10 model-domain settings. Paired seed-level permutation tests confirm these differences after Holm correction (10,000 permutations. All
$p_{\mathrm{Holm}}\le .002$) in Appendix~\ref{app:selector-significance}.
The comparison with PRC also shows that response-level tail control does not provide the same protection as directly optimizing performance across metric--group pairs.
These improvements in worst-case and Worst 25\% quality do not come at a high cost to overall mean quality. Across all 10 model--domain settings, the Constrained Mixed Overall Mean remains within 0.026 of the better single-system-prompt result and is lower in three settings.

\subsubsection{System-Prompt Weight Analysis}
Figure~\ref{fig:prompt-weight} shows the weights assigned to individual system prompts by Constrained Mixed GroupDRO for each model and domain.

\begin{figure}[t]
    \centering
    \includegraphics[width=\linewidth]{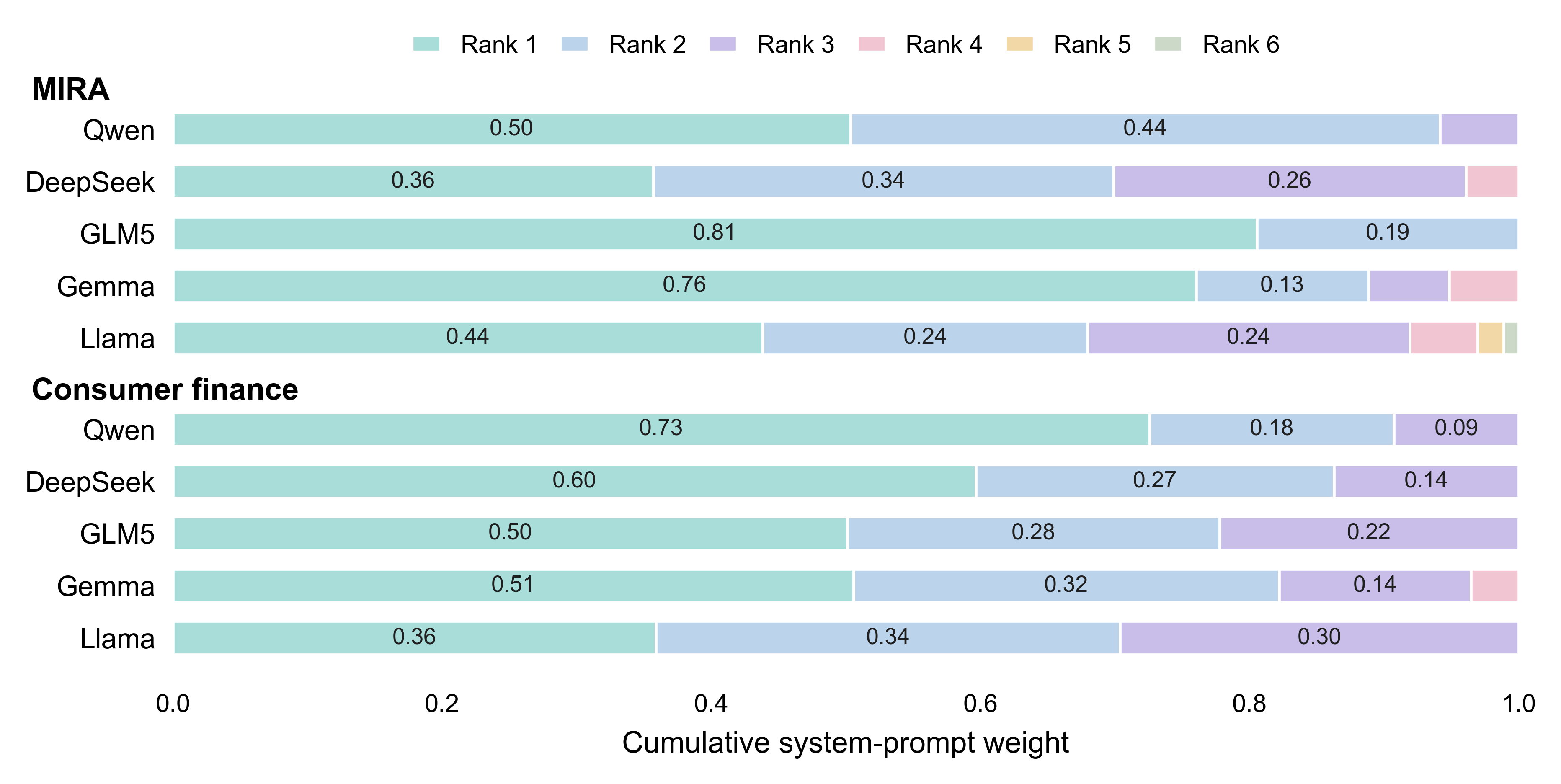}
    \caption{System-prompt weights learned by Constrained Mixed
    GroupDRO. }
    \label{fig:prompt-weight}
\end{figure}

Constrained Mixed GroupDRO uses only a few system prompts for each model: 2-6 in MIRA and 3-4 in consumer finance. The top 2 prompts together receive 0.68-1.00 of the weight in MIRA and 0.70-0.91 in consumer finance. MIRA GLM5 uses only two prompts, with one performing better on $m_1$ and $m_2$ and the other on $m_3$. MIRA Llama uses six prompts because actionability, completeness, and no-referral prompts help different metric-group pairs. Of these six prompts, the top three receive 91.9\% of the total weight.

The family-level analysis in Appendix~\ref{app:family} shows that completeness prompts receive the largest weights in MIRA, while actionability prompts receive the largest weights in consumer finance. For finance Gemma, removing the highest-weight prompt increases development worst case loss by 0.070, while removing the lowest-weight prompt increases it by only 0.002.

To illustrate how the selected prompts affect individual answers, we also present case studies using one held-out example from each domain in Appendix~\ref{app:case}, where the highest-weight Constrained Mixed prompt produces the best-scoring response in both cases.

\subsubsection{Candidate-Growth Analysis}
Our main experiments use a pool of 55 system prompts, but the number of available candidates may vary in practice. We therefore examine performance as the pool size increases from $K=5$ to $K=55$, following the candidate-growth protocol described in Section~\ref{sec:pro}. Figure~\ref{fig:candidate-growth-test} reports the test Overall Mean and Worst 25\% Mean for consumer finance and MIRA.

On test data, Constrained Mixed achieves the lowest Worst 25\% Mean across all pool sizes in both domains. As the pool size increases, its Worst 25\% Mean generally decreases, while its Overall Mean remains relatively stable. Scores often stabilize after about 25 or 35 system prompts, suggesting that moderate-size pools already provide most of the improvement. The corresponding development and worst-case results show similar patterns and are reported in Appendix~\ref{app:candidate-growth}.

\begin{figure}[t]
    \centering
    \includegraphics[width=\linewidth]{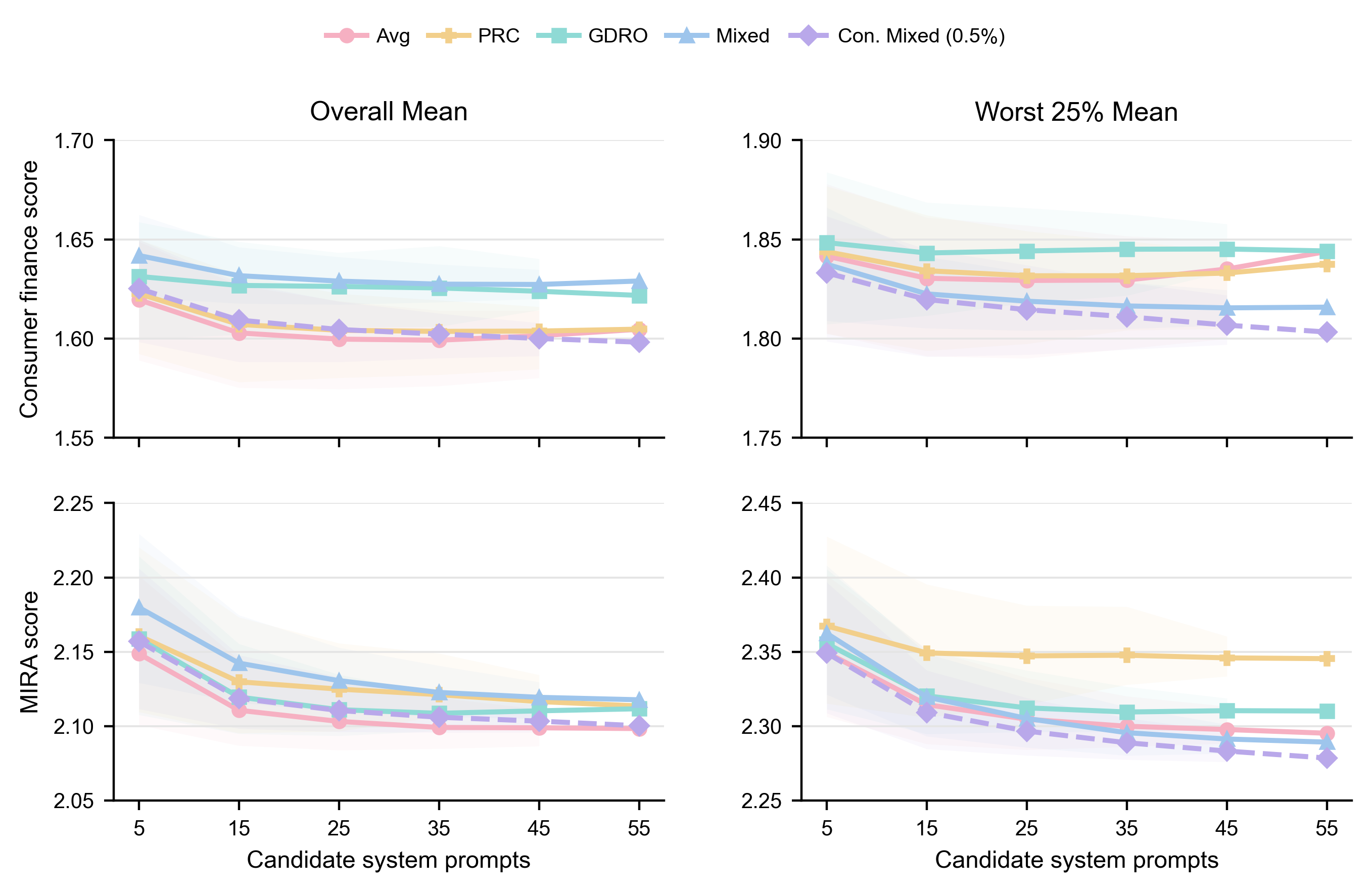}
    \caption{Test candidate-growth results for consumer finance and MIRA}
    \label{fig:candidate-growth-test}
\end{figure}

\section{Conclusion}
Equivalent questions can receive LLM answers of different quality depending on how they are phrased. We show that this problem can be mitigated at the system-prompt selection stage. Constrained Mixed GroupDRO assigns weights to system prompts in an existing pool to improve joint worst-case quality across evaluation metrics and groups while keeping average quality nearly unchanged. Our results show that different system prompts help different parts of the evaluation set and that moderate-size pools provide much of the full-pool benefit. Overall, the method offers a practical way to make LLM information-seeking systems more robust to differences in question phrasing.

\bibliography{aaai2027}

@article{vonneumann1928,
  author  = {v. Neumann, J.},
  title   = {Zur Theorie der Gesellschaftsspiele},
  journal = {Mathematische Annalen},
  year    = {1928},
  volume  = {100},
  number  = {1},
  pages   = {295--320},
  doi     = {10.1007/BF01448847},
  url     = {https://doi.org/10.1007/BF01448847}
}

@inproceedings{sessa2020mixed,
  title={Mixed strategies for robust optimization of unknown objectives},
  author={Sessa, Pier Giuseppe and Bogunovic, Ilija and Kamgarpour, Maryam and Krause, Andreas},
  booktitle={International Conference on Artificial Intelligence and Statistics},
  pages={2970--2980},
  year={2020},
  organization={PMLR}
}

@inproceedings{sagawa2020dro,
title={Distributionally Robust Neural Networks},
author={Shiori Sagawa and Pang Wei Koh and Tatsunori B. Hashimoto and Percy Liang},
booktitle={International Conference on Learning Representations},
year={2020},
url={https://openreview.net/forum?id=ryxGuJrFvS}
}

@inproceedings{diana2021minimax,
author = {Diana, Emily and Gill, Wesley and Kearns, Michael and Kenthapadi, Krishnaram and Roth, Aaron},
title = {Minimax Group Fairness: Algorithms and Experiments},
year = {2021},
isbn = {9781450384735},
publisher = {Association for Computing Machinery},
address = {New York, NY, USA},
url = {https://doi.org/10.1145/3461702.3462523},
doi = {10.1145/3461702.3462523},
booktitle = {Proceedings of the 2021 AAAI/ACM Conference on AI, Ethics, and Society},
pages = {66–76},
numpages = {11},
location = {Virtual Event, USA},
series = {AIES '21}
}

@misc{xu2026,
      title={MIRA: A Bilingual Benchmark for Medical Information Response Audit}, 
      author={Mengyu Xu and Qiaoxin Yang and Qianqian Wang and Xiwei Dai and Weiyi Wu and Chongyang Gao},
      year={2026},
      eprint={2605.28025},
      archivePrefix={arXiv},
      primaryClass={cs.AI},
      url={https://arxiv.org/abs/2605.28025}, 
}

@misc{hashimoto2018,
      title={Fairness Without Demographics in Repeated Loss Minimization}, 
      author={Tatsunori B. Hashimoto and Megha Srivastava and Hongseok Namkoong and Percy Liang},
      year={2018},
      eprint={1806.08010},
      archivePrefix={arXiv},
      primaryClass={stat.ML},
      url={https://arxiv.org/abs/1806.08010}, 
}

@inproceedings{zhou2023,
title={Large Language Models are Human-Level Prompt Engineers},
author={Yongchao Zhou and Andrei Ioan Muresanu and Ziwen Han and Keiran Paster and Silviu Pitis and Harris Chan and Jimmy Ba},
booktitle={The Eleventh International Conference on Learning Representations },
year={2023},
url={https://openreview.net/forum?id=92gvk82DE-}
}

@inproceedings{zollo2024,
title={Prompt Risk Control: A Rigorous Framework for Responsible Deployment of Large Language Models},
author={Thomas P Zollo and Todd Morrill and Zhun Deng and Jake Snell and Toniann Pitassi and Richard Zemel},
booktitle={The Twelfth International Conference on Learning Representations},
year={2024},
url={https://openreview.net/forum?id=5tGGWOijvq}
}

@article{hofmann2024,
  title={AI generates covertly racist decisions about people based on their dialect},
  author={Hofmann, Valentin and Kalluri, Pratyusha Ria and Jurafsky, Dan and King, Sharese},
  journal={Nature},
  volume={633},
  number={8028},
  pages={147--154},
  year={2024},
  doi={10.1038/s41586-024-07856-5},
  url={https://doi.org/10.1038/s41586-024-07856-5}
}

@article{berkman2011,
  title={Low health literacy and health outcomes: an updated systematic review},
  author={Berkman, N. D. and Sheridan, S. L. and Donahue, K. E. and Halpern, D. J. and Crotty, K},
  journal={Annals of Internal Medicine},
  volume={155},
  number={2},
  pages={97--107},
  year={2011},
  doi={10.7326/0003-4819-155-2-201107190-00005}
}

@inproceedings{gupta2024,
title={Bias Runs Deep: Implicit Reasoning Biases in Persona-Assigned {LLM}s},
author={Shashank Gupta and Vaishnavi Shrivastava and Ameet Deshpande and Ashwin Kalyan and Peter Clark and Ashish Sabharwal and Tushar Khot},
booktitle={The Twelfth International Conference on Learning Representations},
year={2024},
url={https://openreview.net/forum?id=kGteeZ18Ir}
}

@inproceedings{yang2024large,
  title={Large language models as optimizers},
  author={Yang, Chengrun and Wang, Xuezhi and Lu, Yifeng and Liu, Hanxiao and Le, Quoc V and Zhou, Denny and Chen, Xinyun},
  booktitle={International Conference on Learning Representations},
  year={2024}
}

@article{fernando2023promptbreeder,
  title={Promptbreeder: Self-referential self-improvement via prompt evolution},
  author={Fernando, Chrisantha and Banarse, Dylan and Michalewski, Henryk and Osindero, Simon and Rockt{\"a}schel, Tim},
  journal={arXiv preprint arXiv:2309.16797},
  year={2023}
}

@inproceedings{agrawal2026gepa,
title={{GEPA}: Reflective Prompt Evolution Can Outperform Reinforcement Learning},
author={Lakshya A Agrawal and Shangyin Tan and Dilara Soylu and Noah Ziems and Rishi Khare and Krista Opsahl-Ong and Arnav Singhvi and Herumb Shandilya and Michael J Ryan and Meng Jiang and Christopher Potts and Koushik Sen and Alex Dimakis and Ion Stoica and Dan Klein and Matei Zaharia and Omar Khattab},
booktitle={The Fourteenth International Conference on Learning Representations},
year={2026}
}

@inproceedings{li2021prefix,
    title = "Prefix-Tuning: Optimizing Continuous Prompts for Generation",
    author = "Li, Xiang Lisa  and
      Liang, Percy",
    editor = "Zong, Chengqing  and
      Xia, Fei  and
      Li, Wenjie  and
      Navigli, Roberto",
    booktitle = "Proceedings of the 59th Annual Meeting of the Association for Computational Linguistics and the 11th International Joint Conference on Natural Language Processing (Volume 1: Long Papers)",
    month = aug,
    year = "2021",
    address = "Online",
    publisher = "Association for Computational Linguistics",
    url = "https://aclanthology.org/2021.acl-long.353/",
    doi = "10.18653/v1/2021.acl-long.353",
    pages = "4582--4597"
}

@inproceedings{lester2021power,
  title={The Power of Scale for Parameter-Efficient Prompt Tuning},
  author={Lester, Brian and Al-Rfou, Rami and Constant, Noah},
  booktitle={Proceedings of the 2021 conference on empirical methods in natural language processing},
  pages={3045--3059},
  year={2021}
}

@misc{liu2023gpt,
      title={GPT Understands, Too}, 
      author={Xiao Liu and Yanan Zheng and Zhengxiao Du and Ming Ding and Yujie Qian and Zhilin Yang and Jie Tang},
      year={2023},
      eprint={2103.10385},
      archivePrefix={arXiv},
      primaryClass={cs.CL},
      url={https://arxiv.org/abs/2103.10385}, 
}

@inproceedings{shin2020autoprompt,
  title={Autoprompt: Eliciting knowledge from language models with automatically generated prompts},
  author={Shin, Taylor and Razeghi, Yasaman and Iv, Robert L Logan and Wallace, Eric and Singh, Sameer},
  booktitle={Proceedings of the 2020 conference on empirical methods in natural language processing (EMNLP)},
  pages={4222--4235},
  year={2020}
}

@article{wen2023hard,
  title={Hard prompts made easy: Gradient-based discrete optimization for prompt tuning and discovery},
  author={Wen, Yuxin and Jain, Neel and Kirchenbauer, John and Goldblum, Micah and Geiping, Jonas and Goldstein, Tom},
  journal={Advances in Neural Information Processing Systems},
  volume={36},
  pages={51008--51025},
  year={2023}
}

@inproceedings{sun2022black,
  title={Black-box tuning for language-model-as-a-service},
  author={Sun, Tianxiang and Shao, Yunfan and Qian, Hong and Huang, Xuanjing and Qiu, Xipeng},
  booktitle={International Conference on Machine Learning},
  pages={20841--20855},
  year={2022},
  organization={PMLR}
}

@inproceedings{deng2022rlprompt,
    title = "{RLP}rompt: Optimizing Discrete Text Prompts with Reinforcement Learning",
    author = "Deng, Mingkai  and
      Wang, Jianyu  and
      Hsieh, Cheng-Ping  and
      Wang, Yihan  and
      Guo, Han  and
      Shu, Tianmin  and
      Song, Meng  and
      Xing, Eric  and
      Hu, Zhiting",
    editor = "Goldberg, Yoav  and
      Kozareva, Zornitsa  and
      Zhang, Yue",
    booktitle = "Proceedings of the 2022 Conference on Empirical Methods in Natural Language Processing",
    month = dec,
    year = "2022",
    address = "Abu Dhabi, United Arab Emirates",
    publisher = "Association for Computational Linguistics",
    url = "https://aclanthology.org/2022.emnlp-main.222/",
    doi = "10.18653/v1/2022.emnlp-main.222",
    pages = "3369--3391"
}

@misc{prasad2023grips,
      title={GrIPS: Gradient-free, Edit-based Instruction Search for Prompting Large Language Models}, 
      author={Archiki Prasad and Peter Hase and Xiang Zhou and Mohit Bansal},
      year={2023},
      eprint={2203.07281},
      archivePrefix={arXiv},
      primaryClass={cs.CL},
      url={https://arxiv.org/abs/2203.07281}, 
}

@inproceedings{guo2024evo,
  title={Connecting Large Language Models with Evolutionary Algorithms Yields Powerful Prompt Optimizers},
  author={Guo, Qingyan and Wang, Rui and Guo, Junliang and Li, Bei and Song, Kaitao and Tan, Xu and Liu, Guoqing and Bian, Jiang and Yang, Yujiu},
  booktitle={International Conference on Learning Representations},
  volume={2024},
  pages={34133--34156},
  year={2024}
}

@inproceedings{opsahl-ong2024,
    title = "Optimizing Instructions and Demonstrations for Multi-Stage Language Model Programs",
    author = "Opsahl-Ong, Krista  and
      Ryan, Michael J  and
      Purtell, Josh  and
      Broman, David  and
      Potts, Christopher  and
      Zaharia, Matei  and
      Khattab, Omar",
    editor = "Al-Onaizan, Yaser  and
      Bansal, Mohit  and
      Chen, Yun-Nung",
    booktitle = "Proceedings of the 2024 Conference on Empirical Methods in Natural Language Processing",
    month = nov,
    year = "2024",
    address = "Miami, Florida, USA",
    publisher = "Association for Computational Linguistics",
    url = "https://aclanthology.org/2024.emnlp-main.525/",
    doi = "10.18653/v1/2024.emnlp-main.525",
    pages = "9340--9366"
}

@misc{yang_qwen3_2025,
  title={{Qwen3} Technical Report},
  author={Yang, An and Li, Anfeng and Yang, Baosong and Zhang, Beichen and Hui, Binyuan and Zheng, Bo and Yu, Bowen and Gao, Chang and Huang, Chengen and Lv, Chenxu and Zheng, Chujie and Liu, Dayiheng and Zhou, Fan and Huang, Fei and Hu, Feng and Ge, Hao and Wei, Haoran and Lin, Huan and Tang, Jialong and Yang, Jian and Tu, Jianhong and Zhang, Jianwei and Yang, Jianxin and Yang, Jiaxi and Zhou, Jing and Zhou, Jingren and Lin, Junyang and Dang, Kai and Bao, Keqin and Yang, Kexin and Yu, Le and Deng, Lianghao and Li, Mei and Xue, Mingfeng and Li, Mingze and Zhang, Pei and Wang, Peng and Zhu, Qin and Men, Rui and Gao, Ruize and Liu, Shixuan and Luo, Shuang and Li, Tianhao and Tang, Tianyi and Yin, Wenbiao and Ren, Xingzhang and Wang, Xinyu and Zhang, Xinyu and Ren, Xuancheng and Fan, Yang and Su, Yang and Zhang, Yichang and Zhang, Yinger and Wan, Yu and Liu, Yuqiong and Wang, Zekun and Cui, Zeyu and Zhang, Zhenru and Zhou, Zhipeng and Qiu, Zihan},
  year={2025},
  eprint={2505.09388},
  archivePrefix={arXiv},
  primaryClass={cs.CL},
  doi={10.48550/arXiv.2505.09388},
  url={https://arxiv.org/abs/2505.09388}
}

@techreport{deepseek2026deepseekv4,
  title = {{DeepSeek-V4}: Towards Highly Efficient Million-Token Context Intelligence},
  author = {{DeepSeek AI}},
  institution = {DeepSeek},
  year = {2026},
  month = apr,
  url = {https://huggingface.co/deepseek-ai/DeepSeek-V4-Pro},
  note = {Technical report. Accessed: 2026-05-20}
}

@misc{openai2026gpt54,
  author = {{OpenAI}},
  title = {Introducing {GPT-5.4}},
  year = {2026},
  month = mar,
  howpublished = {\url{https://openai.com/index/introducing-gpt-5-4/}}
}

@misc{zheng_judging_2023,
	title = {Judging {LLM}-as-a-{Judge} with {MT}-{Bench} and {Chatbot} {Arena}},
	url = {http://arxiv.org/abs/2306.05685},
	doi = {10.48550/arXiv.2306.05685},
	urldate = {2026-05-19},
	publisher = {arXiv},
	author = {Zheng, Lianmin and Chiang, Wei-Lin and Sheng, Ying and Zhuang, Siyuan and Wu, Zhanghao and Zhuang, Yonghao and Lin, Zi and Li, Zhuohan and Li, Dacheng and Xing, Eric P. and Zhang, Hao and Gonzalez, Joseph E. and Stoica, Ion},
	month = dec,
	year = {2023},
	note = {arXiv:2306.05685 [cs.CL]},
}

@misc{glm5team2026,
      title={GLM-5: from Vibe Coding to Agentic Engineering}, 
      author={GLM-5-Team},
      year={2026},
      eprint={2602.15763},
      archivePrefix={arXiv},
      primaryClass={cs.LG},
      url={https://arxiv.org/abs/2602.15763}, 
}

@misc{gemmateam2026,
      title={Gemma 4 Technical Report}, 
      author={Gemma Team and Sherif El Abd and Vaibhav Aggarwal and Robin Algayres and Alek Andreev and Olivier Bachem and Ian Ballantyne and Cormac Brick and Victor Cărbune and Michelle Casbon and Mayank Chaturvedi and Victor Cotruta and Alice Coucke and Phil Culliton and Robert Dadashi and Lucas Dixon and Mohamed Elhawaty and Utku Evci and Clément Farabet and Johan Ferret and Filippo Galgani and Sertan Girgin and Jean-Bastien Grill and Maarten Grootendorst and Jiaxian Guo and Cassidy Hardin and Yanzhang He and Steven M. Hernandez and Omri Homburger and Léonard Hussenot and Juyeong Ji and Armand Joulin and Aishwarya Kamath and Parnian Kassraie and Olivier Lacombe and Preethi Lahoti and Gaël Liu and Gus Martins and Luciano Martins and Tatiana Matejovicova and Ramona Merhej and Nikola Momchev and Sneha Mondal and Ryan Mullins and Sindhu Raghuram Panyam and Shreya Pathak and Sarah Perrin and André Susano Pinto and Etienne Pot and Angéline Pouget and Alexandre Ramé and Sabela Ramos and Douglas Reid and David Rim and Morgane Rivière and Karsten Roth and Louis Rouillard and Omar Sanseviero and Pier Giuseppe Sessa and Shane Settle and Danila Sinopalnikov and Sara Smoot and Piotr Stanczyk and Andreas Steiner and Lawrence Stewart and Ilya Tolstikhin and Michael Tschannen and Anton Tsitsulin and Nino Vieillard and Renjie Wu and Pingmei Xu and Haichuan Yang and Edouard Yvinec and Li Zhang and Joe Zou and Nicolas Aagnes and Abdelrahman Abdelhamed and Shivani Agrawal and Shubham Agrawal and Ibrahim Alabdulmohsin and Jean Baptiste Alayrac and Uri Alon and Chandramouli Amarnath and Ankesh Anand and Chrysovalantis Anastasiou and Setareh Ariafar and François-Xavier Aubet and Kyriakos Axiotis and Federico Barbero and Joelle Barral and Alexei Bendebury and Urs Bergmann and Stanley Bileschi and Kat Black and Mathieu Blondel and Sebastian Borgeaud and Arthur Bražinskas and Ryan Burnell and Robert Busa-Fekete and Mu Cai and Glenn Cameron and Charlotte Caucheteux and Garima Chadha and Jetha Chan and Aditya Chawla and Blake Jianhang Chen and Jesse Chen and Lin Chen and Xu Chen and Derek Cheng and Tzu-hsiang Chien and Nikolai Chinaev and Yi Chou and Zhaohui Chu and Benjamin Coleman and Pooja Consul and Sam Conway-Rahman and Scott Crowell and Dylan Cutler and Vivek Dani and Samira Daruki and Anil Das and Daniel Deutsch and Nishanth Dikkala and Li Ding and Qiuhan Ding and Shenil Dodhia and Konstantin Donhauser and Tulsee Doshi and Anca Dragan and Alex Druinsky and Sahil Dua and Zoltan Egyed and Danielle Eisenbud and Daniel Eppens and Cindy Fan and Bahare Fatemi and Yassir Fathullah and Vlad Feinberg and Milen Ferev and Takumi Fujimoto and Isaac Galatzer-Levy and João Gante and Simon Geisler and Soham Ghosal and Antonious M. Girgis and Alec Go and Alhaad Gokhale and Alex Grills and Yiming Gu and Pramod Gupta and Guru Guruganesh and Raia Hadsell and Hamza Harkous and Jitendra Harlalka and Demis Hassabis and Anja Hauth and Joe Heyward and Arian Hosseini and Chih-Yang Hsia and I-Hung Hsu and Xiaopeng Huang and Yangsibo Huang and Kevin Hui and Adrian Hutter and Te I and Fotis Iliopoulos and Advait Jain and Ganesh Jawahar and Ziwei Ji and Qilin Jin and Melvin Johnson and Kandarp Joshi and Arun Kandoor and Wang-Cheng Kang and Koray Kavukcuoglu and Mehran Kazemi and Kathleen Kenealy and Amr Khalifa and Phoebe Kirk and Suraj Kothawade and Vitaly Kovalev and Neel Kovelamudi and Adam Kraft and Ravin Kumar and Harish Kuppam and Justin Lannin and Chen-Yu Lee and Seungji Lee and Dmitry Lepikhin and Dongdong Li and Qiujia Li and Valentin Liévin and Ethan Lin and Ziqian Lin and Casper Liu and Tianlin Liu and Tianqi Liu and Xin Liu and Mayank Lunayach and Min Ma and Gagan Madan and Andrii Maksai and Eric Malmi and Michal Matuszak and Daniel McDuff and Gaurav Menghani and Daniil Mirylenka and Karolis Misiunas and Vedant Misra and Andreea Mitran and Kareem Mohamed and Maksim Mukha and Eric Noland and James O'Donnell and Kate Olszewska and Bernett Orlando and Wanqiong Pan and Rina Panigrahy and Unnati Parekh and Chunjong Park and Eric Paskie and Liqian Peng and Bryce Petrini and Slav Petrov and Jonas Pfeiffer and Bilal Piot and Martyna Plomecka and Siim Poder and Octavio Ponce and Arijit Pramanik and David Racz and Anish Rajan and Michelle Ramanovich and Anand Rao and Marvin Ritter and Vitor Rodrigues and Evan Rosen and Mikołaj Rybiński and Noveen Sachdeva and Michaël E. Sander and Rohit Sathyanarayana and Sagar Savla and Samuel Schmidgall and Tal Schuster and Benoit Seguin and Andrew Sellergren and Aliaksei Severyn and Izhak Shafran and Dhruv Shah and Yuan Shangguan and Ashish Shenoy and Pradeep Shenoy and Rakesh Shivanna and Pauline Sho and Lucas Spangher and Wojciech Stokowiec and Tim Strother and Yao Su and Yinghao Sun and Mukund Sundararajan and Andrea Tacchetti and Mor Hazan Taege and Pouya Tafti and Chetan Tekur and Rahul Thapa and Madeleine Traverse and Lenart Treven and Tao Tu and Chien Te Tung and Petar Veličković and Malini Pooni Venkat and Sagar Gubbi Venkatesh and Vidya Venkiteswaran and Francesco Visin and Alex Vitvitskyi and Kiran Vodrahalli and Weiyi Wang and Xin Wang and Tris Warkentin and Jan Wassenberg and John Wieting and Lechao Xiao and Hao Xu and Yuhui Xu and Fuzhao Xue and Arun Yadav and Jun Yan and Antoine Yang and Lin Yang and Ming-Hsuan Yang and Ziyu Ying and Jae Hyeon Yoo and Sajjad Zafar and Fred Zhang and Jiageng Zhang and Jianyi Zhang and Xiaofan Zhang and Chao Zhao and David Zhou and Chen Zou},
      year={2026},
      eprint={2607.02770},
      archivePrefix={arXiv},
      primaryClass={cs.CL},
      url={https://arxiv.org/abs/2607.02770}, 
}

@misc{llama4,
  author = {{Meta AI}},
  title  = {The {Llama 4} Herd: The Beginning of a New Era of Natively Multimodal {AI} Innovation},
  year   = {2025},
  url    = {https://ai.meta.com/blog/llama-4-multimodal-intelligence/},
  note   = {Accessed: 2026-07-15}
}

@misc{gpt54mini,
  author       = {{OpenAI}},
  title        = {Introducing {GPT-5.4} mini and nano},
  year         = {2026},
  month        = mar,
  howpublished = {\url{https://openai.com/index/introducing-gpt-5-4-mini-and-nano/}}
}

@inproceedings{do2025automatic,
  title={Automatic prompt selection for large language models},
  author={Do, Viet-Tung and Nguyen, Xuan-Quang and Hoang, Van-Khanh and Nguyen, Duy-Hung and Sabahi, Shahab and Yang, Jeff and Hotta, Hajime and Nguyen, Minh-Tien and Le, Hung},
  booktitle={Pacific-Asia Conference on Knowledge Discovery and Data Mining},
  pages={91--102},
  year={2025},
  organization={Springer}
}

\clearpage
\appendix
\section{Extended Related Work}
\label{app:relatedwork}
Prompt optimization shows that prompt design can substantially affect LLM output quality. White-box methods use parameter or gradient access to optimize soft prompts~\citep{li2021prefix,lester2021power,liu2023gpt} or search discrete tokens~\citep{shin2020autoprompt,wen2023hard}. Black-box methods search with model feedback only, via derivative-free optimization~\citep{sun2022black}, reinforcement learning~\citep{deng2022rlprompt}, or edit-based search~\citep{prasad2023grips}, and recent LLM-as-optimizer and evolutionary methods generate prompt candidates automatically~\citep{zhou2023,yang2024large,fernando2023promptbreeder,guo2024evo,agrawal2026gepa,opsahl-ong2024}. A complementary line selects prompts from a candidate set rather than generating new text, ranking candidates by validation performance or learned evaluators~\citep{do2025automatic}.

\section{Proof of Proposition~\ref{prop:collapse}}
\label{app:proof}
\begin{proof}
Each feasible single system prompt $p_i\in\mathcal P_\epsilon$ corresponds to the one-hot vector $e_i\in\Delta_N$, which places all weight on $p_i$. By the definition of $\mathcal P_\epsilon$,
\[
\bar L(e_i) = \frac{1}{|\mathcal M||\mathcal G|}\sum_{m\in\mathcal M}\sum_{g\in\mathcal G}R_{m,g}(e_i) \le (1+\epsilon)\bar L_{\mathrm{avg}}.
\]
Therefore, $e_i\in\Delta_{N,\epsilon}$. Thus, the one-hot vectors corresponding to feasible single-system-prompt choices form a subset of the constrained mixed feasible set, and minimizing over the larger set gives $V_{\mathrm{mix},\epsilon}\le
V_{\mathrm{pure},\epsilon}$.

Both feasible sets are nonempty. Because the losses are
nonnegative and $\epsilon\ge0$, 
\[
\bar L_{\mathrm{avg}} \le (1+\epsilon)\bar L_{\mathrm{avg}}.
\]
So,
$p_{\mathrm{avg}}\in\mathcal P_\epsilon$ and its corresponding one-hot vector belongs to $\Delta_{N,\epsilon}$.

Because $\mathcal P_\epsilon$ is finite and nonempty, a constrained single-prompt minimizer exists. The set $\Delta_{N,\epsilon}$ is a closed subset of the compact simplex $\Delta_N$ and is therefore compact. Since 
\[
w\mapsto \max_{\substack{m\in\mathcal M\\g\in\mathcal G}}R_{m,g}(w)
\]
is continuous, an optimal constrained mixed solution also exists.

First suppose
$V_{\mathrm{mix},\epsilon}=V_{\mathrm{pure},\epsilon}$. Let
\[
p_{i^\star}\in \argmin_{p_i\in\mathcal P_\epsilon}\max_{\substack{m\in\mathcal M\\g\in\mathcal G}}R_{m,g}(e_i)
\]
be a constrained single-system-prompt minimizer. Then $e_{i^\star}\in\Delta_{N,\epsilon}$, and its objective value is
\[
\max_{\substack{m\in\mathcal M\\g\in\mathcal G}}R_{m,g}(e_{i^\star}) = V_{\mathrm{pure},\epsilon} = V_{\mathrm{mix},\epsilon}.
\]
Thus, $e_{i^\star}$ is an optimal constrained mixed solution that places all weight on a single prompt in $\mathcal P_\epsilon$.

Conversely, suppose an optimal constrained mixed solution is one-hot, say $e_i$. Its feasibility implies $p_i\in\mathcal P_\epsilon$, since $\bar L(e_i)$ equals the mean loss of prompt $p_i$. Therefore,
\[
V_{\mathrm{pure},\epsilon} \le \max_{\substack{m\in \mathcal M\\ g\in\mathcal G}}R_{m,g}(e_i) = V_{\mathrm{mix},\epsilon}.
\]
Together with $V_{\mathrm{mix},\epsilon}\le V_{\mathrm{pure},\epsilon}$, this gives $V_{\mathrm{mix},\epsilon} = V_{\mathrm{pure},\epsilon}$.

Finally, suppose $V_{\mathrm{mix},\epsilon} < V_{\mathrm{pure},\epsilon}$. Every feasible one-hot vector $e_i$ corresponds to some $p_i\in\mathcal P_\epsilon$ and therefore has objective value 
\[
\max_{\substack{m\in\mathcal M\\g\in\mathcal G}}R_{m,g}(e_i) \ge V_{\mathrm{pure},\epsilon} > V_{\mathrm{mix},\epsilon}.
\]

Hence, no feasible one-hot vector can be optimal. Since an optimal constrained mixed solution exists and cannot be one-hot, every optimal constrained mixed solution places positive weight on at least two system prompts.
\end{proof}

\section{Consumer-Finance Scoring Agreement}
\label{app:finance-agreement}
Table~\ref{tab:finance-agreement} reports agreement between the trained finance annotator and the LLM judge. Exact agreement is the proportion of identical ratings, adjacent agreement is the proportion that differ by at most one point, and QWK denotes quadratic weighted kappa.

\begin{table}[t]
\centering
\begin{tabular}{lrrrr}
\toprule
\textbf{Metric} & $n$ & \textbf{Exact}
& \textbf{Adjacent} & \textbf{QWK} \\
\midrule
$m_1$ & 250 & 0.804 & 0.964 & 0.832 \\
$m_2$ & 250 & 0.904 & 0.988 & 0.923 \\
$m_3$ & 250 & 0.924 & 0.996 & 0.918 \\
\bottomrule
\end{tabular}
\caption{Consumer-finance scoring agreement on the
250-response audit set.}
\label{tab:finance-agreement}
\end{table}

\section{Full Group Validity Results}
This section provides the complete results supporting the group-validity analysis in Section~\ref{sec:validity}. We first report the baseline worst-minus-mean gaps and then present the full fixed-effect regression results for each domain.

\subsection{Baseline Worst-Minus-Mean Gaps}
\label{app:gaps}
Table~\ref{tab:disparity} reports the baseline worst-minus-mean gap for each model, domain, and evaluation metric. Each gap is calculated by subtracting the mean loss across the 24 evaluation groups from the highest group loss. Larger gaps indicate greater variation in response quality across groups. 

For example, in consumer finance, Llama4Scout has gaps of 0.607 on $m_1$ and 0.624 on $m_2$. In MIRA, DeepSeek-V4Pro has a gap of 0.280 on $m_2$, and Llama4Scout has a gap of 0.284 on $m_3$. The gaps are positive across all model-domain-metric settings, showing that the worst-performing group consistently has higher loss than the group mean. This provides additional motivation for optimizing worst-group performance rather than relying only on average quality.

\begin{table}[t]
\centering
{\small
\setlength{\tabcolsep}{1.2mm}
\begin{tabular}{@{}lccc@{}}
\toprule
Model & $m_1$ gap & $m_2$ gap & $m_3$ gap \\
\midrule
\multicolumn{4}{@{}l}{\textit{Consumer finance}} \\
Qwen36Plus     & 0.292 & 0.386 & 0.158 \\
DeepSeek-V4Pro & 0.560 & 0.466 & 0.441 \\
GLM5           & 0.390 & 0.316 & 0.264 \\
Gemma4-31B     & 0.362 & 0.326 & 0.237 \\
Llama4Scout    & 0.607 & 0.624 & 0.275 \\
\midrule
\multicolumn{4}{@{}l}{\textit{MIRA}} \\
Qwen36Plus     & 0.218 & 0.219 & 0.274 \\
DeepSeek-V4Pro & 0.139 & 0.280 & 0.155 \\
GLM5           & 0.198 & 0.145 & 0.146 \\
Gemma4-31B     & 0.234 & 0.176 & 0.133 \\
Llama4Scout    & 0.252 & 0.273 & 0.284 \\
\bottomrule
\end{tabular}
}
\caption{Baseline worst-minus-mean gaps by model, domain, and
evaluation metric.}
\label{tab:disparity}
\end{table}

\subsection{Controlled Group Differences}
\label{app:group-validity}
Tables~\ref{tab:appendix-mira-regression} and \ref{tab:appendix-finance-regression} report the full regression results for MIRA and consumer finance. The \textit{Pooled} rows pool all five models, while the remaining rows report model-specific results. The regressions control for seed question and system prompt, and the pooled regressions also control for the response-generating model. Standard errors are clustered by seed question. Positive coefficients indicate worse outcomes relative to the reference group.

Together, these results show that the evaluation groups capture meaningful variation in response quality. Literacy signals produce the strongest and most consistent differences, while language, question structure, resource constraints, and question framing show domain-specific effects. This supports the use of these groups for joint worst-case system-prompt selection.

\begin{table*}[t]
\centering
{\small
\setlength{\tabcolsep}{1mm}
\begin{tabular}{@{}llccc@{}}
\toprule
\textbf{Model} & \textbf{Predictor}
& $\mathbf{m_1}$ & $\mathbf{m_2}$ & $\mathbf{m_3}$ \\
\midrule
Pooled & Chinese (vs. English) & $-0.096^{*}$ (0.040) & $-0.055$ (0.046) & $-0.049$ (0.038) \\
 & Low HLS (vs. High HLS) & $0.118^{***}$ (0.032) & $0.114^{***}$ (0.031) & $0.026$ (0.033) \\
 & Colloquial (vs. Formal) & $0.022$ (0.016) & $0.010$ (0.014) & $-0.007$ (0.018) \\
 & Skeleton S2 (vs. S1) & $0.011$ (0.007) & $0.001$ (0.005) & $0.006$ (0.006) \\
 & Skeleton S3 (vs. S1) & $0.028^{*}$ (0.011) & $0.014$ (0.011) & $0.012$ (0.010) \\
\midrule
Qwen36Plus & Chinese (vs. English) & $-0.149^{***}$ (0.043) & $-0.094^{*}$ (0.048) & $-0.090^{*}$ (0.036) \\
 & Low HLS (vs. High HLS) & $0.082^{*}$ (0.034) & $0.080^{*}$ (0.031) & $0.016$ (0.029) \\
 & Colloquial (vs. Formal) & $-0.006$ (0.020) & $-0.013$ (0.017) & $-0.018$ (0.021) \\
 & Skeleton S2 (vs. S1) & $0.010$ (0.011) & $0.001$ (0.006) & $0.007$ (0.007) \\
 & Skeleton S3 (vs. S1) & $-0.003$ (0.013) & $-0.022^{\dagger}$ (0.012) & $-0.012$ (0.011) \\
\midrule
DeepSeek-V4Pro & Chinese (vs. English) & $-0.080^{\dagger}$ (0.046) & $-0.055$ (0.052) & $-0.115^{*}$ (0.048) \\
 & Low HLS (vs. High HLS) & $0.116^{**}$ (0.038) & $0.088^{*}$ (0.034) & $0.005$ (0.041) \\
 & Colloquial (vs. Formal) & $0.023$ (0.018) & $0.010$ (0.018) & $-0.003$ (0.022) \\
 & Skeleton S2 (vs. S1) & $0.021^{**}$ (0.008) & $0.011$ (0.008) & $0.018^{\dagger}$ (0.010) \\
 & Skeleton S3 (vs. S1) & $0.033^{*}$ (0.016) & $0.007$ (0.011) & $0.001$ (0.011) \\
\midrule
GLM5 & Chinese (vs. English) & $-0.187^{***}$ (0.055) & $-0.150^{*}$ (0.060) & $-0.110^{*}$ (0.051) \\
 & Low HLS (vs. High HLS) & $0.149^{***}$ (0.039) & $0.147^{***}$ (0.036) & $0.061$ (0.039) \\
 & Colloquial (vs. Formal) & $0.028$ (0.022) & $0.013$ (0.018) & $-0.010$ (0.022) \\
 & Skeleton S2 (vs. S1) & $0.015^{\dagger}$ (0.009) & $0.006$ (0.008) & $0.000$ (0.009) \\
 & Skeleton S3 (vs. S1) & $0.042^{**}$ (0.015) & $0.025^{\dagger}$ (0.014) & $0.008$ (0.013) \\
\midrule
Gemma4-31B & Chinese (vs. English) & $-0.112^{**}$ (0.043) & $-0.095^{\dagger}$ (0.052) & $-0.137^{**}$ (0.047) \\
 & Low HLS (vs. High HLS) & $0.128^{***}$ (0.035) & $0.120^{***}$ (0.034) & $0.018$ (0.035) \\
 & Colloquial (vs. Formal) & $0.039$ (0.024) & $0.021$ (0.018) & $-0.020$ (0.021) \\
 & Skeleton S2 (vs. S1) & $0.007$ (0.010) & $-0.011$ (0.010) & $-0.002$ (0.011) \\
 & Skeleton S3 (vs. S1) & $0.034^{*}$ (0.013) & $0.025^{\dagger}$ (0.015) & $0.020$ (0.013) \\
\midrule
Llama4Scout & Chinese (vs. English) & $0.048$ (0.039) & $0.118^{*}$ (0.048) & $0.210^{***}$ (0.044) \\
 & Low HLS (vs. High HLS) & $0.117^{**}$ (0.037) & $0.135^{***}$ (0.038) & $0.033$ (0.036) \\
 & Colloquial (vs. Formal) & $0.025^{\dagger}$ (0.014) & $0.020$ (0.016) & $0.014$ (0.017) \\
 & Skeleton S2 (vs. S1) & $0.001$ (0.008) & $-0.001$ (0.009) & $0.005$ (0.010) \\
 & Skeleton S3 (vs. S1) & $0.035^{***}$ (0.009) & $0.034^{*}$ (0.015) & $0.042^{**}$ (0.015) \\
\midrule
\multicolumn{5}{@{}l@{}}{\textit{Note.} Positive coefficients indicate worse outcomes.}\\
\multicolumn{5}{@{}l@{}}{Standard errors clustered by seed question are shown in parentheses.}\\
\multicolumn{5}{@{}l@{}}{$^{\dagger}p<.10$, $^{*}p<.05$, $^{**}p<.01$, and $^{***}p<.001$.}\\
\bottomrule
\end{tabular}
}
\caption{Full controlled score differences for MIRA evaluation groups. Reference groups are shown in parentheses. Regressions include seed-question and system-prompt fixed effects; the pooled regressions also include model fixed effects.}
\label{tab:appendix-mira-regression}
\end{table*}

\begin{table*}[t]
\centering
{\small
\setlength{\tabcolsep}{1mm}
\begin{tabular}{@{}llccc@{}}
\toprule
\textbf{Model} & \textbf{Predictor}
& $\mathbf{m_1}$ & $\mathbf{m_2}$ & $\mathbf{m_3}$ \\
\midrule
Pooled & Chinese (vs. English) & $0.006$ (0.143) & $0.025$ (0.132) & $0.003$ (0.082) \\
 & Low FLS (vs. High FLS) & $0.094^{***}$ (0.026) & $0.094^{***}$ (0.022) & $0.060^{***}$ (0.018) \\
 & Low asset (vs. High asset) & $0.017^{*}$ (0.008) & $0.019^{**}$ (0.007) & $-0.010$ (0.007) \\
 & Confusion frame (vs. Direct) & $0.045^{***}$ (0.011) & $0.045^{***}$ (0.010) & $0.041^{***}$ (0.006) \\
 & Misconception frame (vs. Direct) & $0.001$ (0.039) & $-0.028$ (0.031) & $-0.024$ (0.021) \\
\midrule
Qwen36Plus & Chinese (vs. English) & $-0.112$ (0.135) & $-0.108$ (0.124) & $-0.056$ (0.067) \\
 & Low FLS (vs. High FLS) & $0.009$ (0.025) & $0.024$ (0.017) & $0.022$ (0.016) \\
 & Low asset (vs. High asset) & $0.006$ (0.010) & $0.016^{\dagger}$ (0.010) & $-0.004$ (0.009) \\
 & Confusion frame (vs. Direct) & $0.018^{**}$ (0.007) & $0.025^{***}$ (0.006) & $0.016^{**}$ (0.006) \\
 & Misconception frame (vs. Direct) & $-0.022$ (0.035) & $-0.040$ (0.035) & $-0.033^{\dagger}$ (0.019) \\
\midrule
DeepSeek-V4Pro & Chinese (vs. English) & $-0.092$ (0.143) & $-0.061$ (0.132) & $-0.032$ (0.087) \\
 & Low FLS (vs. High FLS) & $0.078^{*}$ (0.034) & $0.085^{**}$ (0.026) & $0.056^{*}$ (0.025) \\
 & Low asset (vs. High asset) & $0.013$ (0.013) & $0.011$ (0.011) & $-0.027^{*}$ (0.012) \\
 & Confusion frame (vs. Direct) & $0.032^{*}$ (0.016) & $0.040^{**}$ (0.012) & $0.036^{***}$ (0.007) \\
 & Misconception frame (vs. Direct) & $0.029$ (0.048) & $0.000$ (0.033) & $0.007$ (0.028) \\
\midrule
GLM5 & Chinese (vs. English) & $-0.065$ (0.146) & $-0.028$ (0.128) & $-0.003$ (0.083) \\
 & Low FLS (vs. High FLS) & $0.055$ (0.036) & $0.054^{\dagger}$ (0.032) & $0.035$ (0.026) \\
 & Low asset (vs. High asset) & $0.030^{*}$ (0.013) & $0.038^{**}$ (0.012) & $-0.001$ (0.010) \\
 & Confusion frame (vs. Direct) & $0.053^{**}$ (0.017) & $0.051^{***}$ (0.011) & $0.056^{***}$ (0.010) \\
 & Misconception frame (vs. Direct) & $0.027$ (0.062) & $-0.005$ (0.047) & $-0.012$ (0.033) \\
\midrule
Gemma4-31B & Chinese (vs. English) & $0.052$ (0.160) & $0.073$ (0.144) & $0.043$ (0.091) \\
 & Low FLS (vs. High FLS) & $0.092^{\dagger}$ (0.049) & $0.100^{*}$ (0.042) & $0.041$ (0.030) \\
 & Low asset (vs. High asset) & $0.003$ (0.018) & $-0.001$ (0.015) & $-0.022^{*}$ (0.011) \\
 & Confusion frame (vs. Direct) & $0.061^{**}$ (0.022) & $0.051^{**}$ (0.018) & $0.051^{***}$ (0.010) \\
 & Misconception frame (vs. Direct) & $-0.014$ (0.060) & $-0.042$ (0.046) & $-0.047$ (0.030) \\
\midrule
Llama4Scout & Chinese (vs. English) & $0.248$ (0.176) & $0.251$ (0.168) & $0.064$ (0.096) \\
 & Low FLS (vs. High FLS) & $0.235^{***}$ (0.064) & $0.206^{***}$ (0.059) & $0.146^{***}$ (0.024) \\
 & Low asset (vs. High asset) & $0.033^{*}$ (0.016) & $0.031^{\dagger}$ (0.016) & $0.006$ (0.010) \\
 & Confusion frame (vs. Direct) & $0.059^{**}$ (0.022) & $0.056^{**}$ (0.021) & $0.044^{***}$ (0.011) \\
 & Misconception frame (vs. Direct) & $-0.014$ (0.034) & $-0.056^{\dagger}$ (0.031) & $-0.034$ (0.022) \\
\midrule
\multicolumn{5}{@{}l@{}}{\textit{Note.} Positive coefficients indicate worse outcomes.}\\
\multicolumn{5}{@{}l@{}}{Standard errors clustered by seed question are shown in parentheses.}\\
\multicolumn{5}{@{}l@{}}{$^{\dagger}p<.10$, $^{*}p<.05$, $^{**}p<.01$, and $^{***}p<.001$.}\\
\bottomrule
\end{tabular}
}
\caption{Full controlled score differences for consumer-finance evaluation groups. Reference groups are shown in parentheses. Regressions include seed-question and system-prompt fixed effects. The pooled regressions also include model fixed effects.}
\label{tab:appendix-finance-regression}
\end{table*}

\subsection{Literacy Placebo Tests}
\label{app:placebo}
As a matched-pair placebo check, we randomly flip the low/high literacy label within matched pairs 1000 times. In MIRA, the observed low-high HLS deltas are 0.118/0.114/0.026 on $m_1/m_2/m_3$, compared with placebo 95th percentiles of 0.005/0.005/0.005. In consumer finance, the observed low-high FLS deltas are 0.094/0.094/0.060, compared with placebo 95th percentiles of 0.007/0.006/0.004. In both domains, the observed deltas exceed all 1000 permutations ($p_{\mathrm{perm}}=1/1001<0.001$), showing that the literacy effects are not reproduced by random relabeling.

\section{Full Development-Set Results}

\subsection{Development Results of Selectors}
\label{app:dev-results}
Tables~\ref{tab:miradev} and~\ref{tab:findev} report the complete development-set results for MIRA and consumer finance. Each table reports the Overall Mean (All), Worst 25\% Mean (25\%), and worst-case score (Worst). Lower values are better.

Constrained Mixed GroupDRO achieves a lower Worst 25\% Mean than Average selection, PRC, and Pure GroupDRO in all 10 model--domain settings. It also achieves a lower worst-case score than all three single-system-prompt selectors in 9 of the 10 settings. The exception is MIRA GLM5, where its score is 2.422 compared with 2.400 for Pure GroupDRO. Its Overall Mean remains within 0.013 of the best single-system-prompt result in every setting. Bold indicates the best result among the single-system-prompt selectors and Constrained Mixed GroupDRO, while underlining indicates that Constrained Mixed GroupDRO outperforms unconstrained Mixed GroupDRO.

\begin{table}[t]
\centering
{\small
\setlength{\tabcolsep}{1mm}
\begin{tabular}{@{}llccccc@{}}
\toprule
Model & Stat. & Avg & PRC & GDRO & Mix. & Con. Mix. \\
\midrule
Qwen
& All   & \textbf{1.860} & 1.912 & 1.878 & 1.889 & \underline{1.870} \\
& 25\%  & 2.061 & 2.114 & 2.056 & 2.029 & \textbf{2.030} \\
& Worst & 2.250 & 2.250 & 2.150 & 2.094 & \textbf{2.107} \\
\midrule

DeepSeek
& All   & \textbf{1.898} & 1.902 & \textbf{1.898} & 1.914 & \underline{1.907} \\
& 25\%  & 2.131 & 2.197 & 2.131 & 2.035 & \textbf{\underline{2.034}} \\
& Worst & 2.250 & 2.300 & 2.250 & 2.062 & \textbf{2.071} \\
\midrule

GLM5
& All   & \textbf{1.999} & 2.085 & 2.051 & 2.048 & \underline{2.009} \\
& 25\%  & 2.256 & 2.331 & 2.281 & 2.225 & \textbf{2.232} \\
& Worst & 2.500 & 2.450 & \textbf{2.400} & 2.292 & 2.422 \\
\midrule

Gemma
& All   & \textbf{2.235} & 2.286 & \textbf{2.235} & 2.319 & \underline{2.247} \\
& 25\%  & 2.442 & 2.611 & 2.442 & 2.465 & \textbf{\underline{2.433}} \\
& Worst & 2.550 & 2.700 & 2.550 & 2.484 & \textbf{2.507} \\
\midrule

Llama
& All   & \textbf{2.719} & 2.736 & \textbf{2.719} & 2.746 & \underline{2.732} \\
& 25\%  & 2.900 & 2.967 & 2.900 & 2.893 & \textbf{\underline{2.888}} \\
& Worst & 3.000 & 3.100 & 3.000 & 2.910 & \textbf{2.936} \\
\bottomrule
\end{tabular}
}
\caption{MIRA development-set results.}
\label{tab:miradev}
\end{table}

\begin{table}[t]
\centering
{\small
\setlength{\tabcolsep}{1mm}
\begin{tabular}{@{}llccccc@{}}
\toprule
Model & Stat. & Avg & PRC & GDRO & Mix. & Con. Mix. \\
\midrule
Qwen
& All   & \textbf{1.129} & \textbf{1.129} & 1.149 & 1.171 & \underline{1.135} \\
& 25\%  & 1.258 & 1.258 & 1.261 & 1.235 & \textbf{1.244} \\
& Worst & 1.400 & 1.400 & 1.350 & 1.237 & \textbf{1.336} \\
\midrule

DeepSeek
& All   & \textbf{1.192} & \textbf{1.192} & 1.201 & 1.204 & \underline{1.198} \\
& 25\%  & 1.294 & 1.294 & 1.294 & 1.271 & \textbf{1.272} \\
& Worst & 1.450 & 1.450 & 1.350 & 1.280 & \textbf{1.297} \\
\midrule

GLM5
& All   & \textbf{1.269} & \textbf{1.269} & 1.286 & 1.329 & \underline{1.275} \\
& 25\%  & 1.414 & 1.414 & 1.444 & 1.379 & \textbf{1.389} \\
& Worst & 1.600 & 1.600 & 1.550 & 1.381 & \textbf{1.492} \\
\midrule

Gemma
& All   & \textbf{1.412} & \textbf{1.412} & 1.435 & 1.457 & \underline{1.419} \\
& 25\%  & 1.628 & 1.628 & 1.564 & 1.523 & \textbf{1.550} \\
& Worst & 1.900 & 1.900 & 1.650 & 1.533 & \textbf{1.612} \\
\midrule

Llama
& All   & \textbf{2.029} & 2.054 & 2.054 & 2.085 & \underline{2.039} \\
& 25\%  & 2.339 & 2.331 & 2.331 & 2.310 &
  \textbf{\underline{2.296}} \\
& Worst & 2.650 & 2.450 & 2.450 & 2.332 & \textbf{2.391} \\
\bottomrule
\end{tabular}
}
\caption{Consumer-finance development-set results. }
\label{tab:findev}
\end{table}

\subsection{Paired Selector Significance Tests}
\label{app:selector-significance}
We use paired seed-level permutation tests to assess the held-out differences between Constrained Mixed GroupDRO and each single-prompt selector. The held-out seed question is the permutation unit. In each of 10,000 permutations, selector labels are swapped within each seed while keeping all 24 question variants and three metric scores together. We then recompute the 72 metric-group means and the three evaluation statistics. We use two-sided permutation $p$-values and apply Holm correction across the nine reported tests. Table~\ref{tab:selector-permutation} reports the complete results.

\begin{table}[t]
\centering
{\small
\setlength{\tabcolsep}{1.2mm}
\begin{tabular}{@{}llrrr@{}}
\toprule
Statistic & Comparator & Mean $\Delta$ & Wins & $p_{\mathrm{Holm}}$ \\
\midrule
Overall Mean & Average selection & +0.002 & 4/10 & $.581$ \\
Overall Mean & Pure GroupDRO & +0.018 & 8/10 & $.006$ \\
Overall Mean & PRC & +0.010 & 5/10 & $.524$ \\
\midrule
Worst 25\% Mean & Average selection & +0.029 & 10/10 & $.002$ \\
Worst 25\% Mean & Pure GroupDRO & +0.036 & 10/10 & $.002$ \\
Worst 25\% Mean & PRC & +0.051 & 10/10 & $<.001$ \\
\midrule
Worst & Average selection & +0.036 & 8/10 & $.114$ \\
Worst & Pure GroupDRO & +0.034 & 8/10 & $.217$ \\
Worst & PRC & +0.063 & 9/10 & $.006$ \\
\bottomrule
\end{tabular}
}
\caption{Paired seed-level permutation tests comparing Constrained Mixed GroupDRO with the single-prompt selectors on held-out data. Positive $\Delta$ means lower loss for Constrained Mixed. Wins count model--domain settings with positive $\Delta$. Holm correction is applied across the nine reported tests.}
\label{tab:selector-permutation}
\end{table}

\subsection{Distribution of Development Score Changes}
\label{app:scorechange}
Table~\ref{tab:overall-score-changes-dev} compares Constrained Mixed GroupDRO with Average selection across all development-set model-question-variant-metric scores. The net changes are $+0.007$ in consumer finance and $+0.011$ in MIRA, showing that the overall development mean changes little in both domains.

\begin{table}[t]
\centering
{\small
\setlength{\tabcolsep}{0.4mm}
\begin{tabular}{@{}lrrrrrr@{}}
\toprule
& \multicolumn{2}{c}{Improved}
& \multicolumn{2}{c}{Worsened}
& \multicolumn{1}{c}{Unch.} & \\
\cmidrule(lr){2-3}
\cmidrule(lr){4-5}
\cmidrule(lr){6-6}
Dev
& \multicolumn{1}{c}{$n$ (\%)}
& \multicolumn{1}{c}{$\bar{\Delta}$}
& \multicolumn{1}{c}{$n$ (\%)}
& \multicolumn{1}{c}{$\bar{\Delta}$}
& \multicolumn{1}{c}{$n$ (\%)}
& \multicolumn{1}{c}{Net $\Delta$} \\
\midrule
Fin.
& 983 (13.7)  & $-0.58$
& 1266 (17.6) & $+0.49$
& 4951 (68.8) & $+0.007$ \\

MIRA
& 1950 (27.1) & $-0.30$
& 2359 (32.8) & $+0.28$
& 2890 (40.1) & $+0.011$ \\
\bottomrule
\end{tabular}
}
\caption{Distribution of development score changes across the full
evaluation set for Constrained Mixed GroupDRO relative to Average
selection. F and M denote consumer finance and MIRA. One MIRA
comparison with an N/A rating is excluded.}
\label{tab:overall-score-changes-dev}
\end{table}

\section{Case studies}
\label{app:case}
\paragraph{Consumer Finance}
We consider a held-out low-financial-literacy question: “What is the difference between a family trust and a regular investment trust, and how do they protect, transfer, and grow money?” 
Average selection chooses a system prompt that emphasizes actionability and asks the model to explain which documents, terms, and numbers the user should check. The response recommends reviewing a trust deed, prospectus, and fees, but does not clearly explain that suitability depends on the family's legal and tax circumstances. It receives scores of $4/4/2$ on $m_1/m_2/m_3$.

Pure GroupDRO chooses a system prompt that emphasizes the preservation of the information needed for judgment and stating important limits and risks. Its response explains the basic purposes of the two trust structures, but omits important legal and return limits and does not provide a clear suitability test. It receives scores of $4/4/3$.

Constrained Mixed GroupDRO places its largest weight ($w=0.506$) on a completeness prompt and distributes the remaining weight across two actionability prompts and one anti-disclaimer/referral-only prompt. The response generated by the highest-weight prompt explains the setup costs, legal and tax conditions, and the information that should be verified. It receives scores of $1/1/2$.

\paragraph{MIRA}
We also consider a held-out low-health-literacy question: ``What should be considered when choosing foods that raise blood sugar slowly for people with type 2 diabetes?'' Average selection chooses a system prompt that emphasizes medical context and completeness. Its response explains that the glycemic index ranges from zero to one hundred, but does not give the cutoff used to identify low-GI foods. It receives scores of $2/2/1$.

Pure GroupDRO chooses a system prompt that explicitly asks the model not to omit criteria, thresholds, or important distinctions. Its response explains that foods with a lower glycemic index cause a more gradual increase in blood glucose, but again omits the low-GI cutoff. It receives scores of $2/2/2$.

The highest-weight prompt in the Constrained Mixed GroupDRO support ($w=0.504$) emphasizes definitions, evaluation criteria, and clinical thresholds. Its response explicitly states that foods with a glycemic index below 55 are classified as low and explains how the threshold, food examples, and portion size relate to post-meal glucose. It receives scores of $1/1/1$.

\section{Additional Candidate-Growth Results}
\label{app:candidate-growth}
Figure~\ref{fig:candidate-growth-dev} reports the development candidate-growth results, while Figure~\ref{fig:candidate-growth-worst} reports the worst-case results on development and held-out test data. Across both domains, performance generally improves as the candidate pool expands, with most gains appearing by $K=25$ or $K=35$.

\begin{figure}[t]
    \centering
    \includegraphics[width=\linewidth]{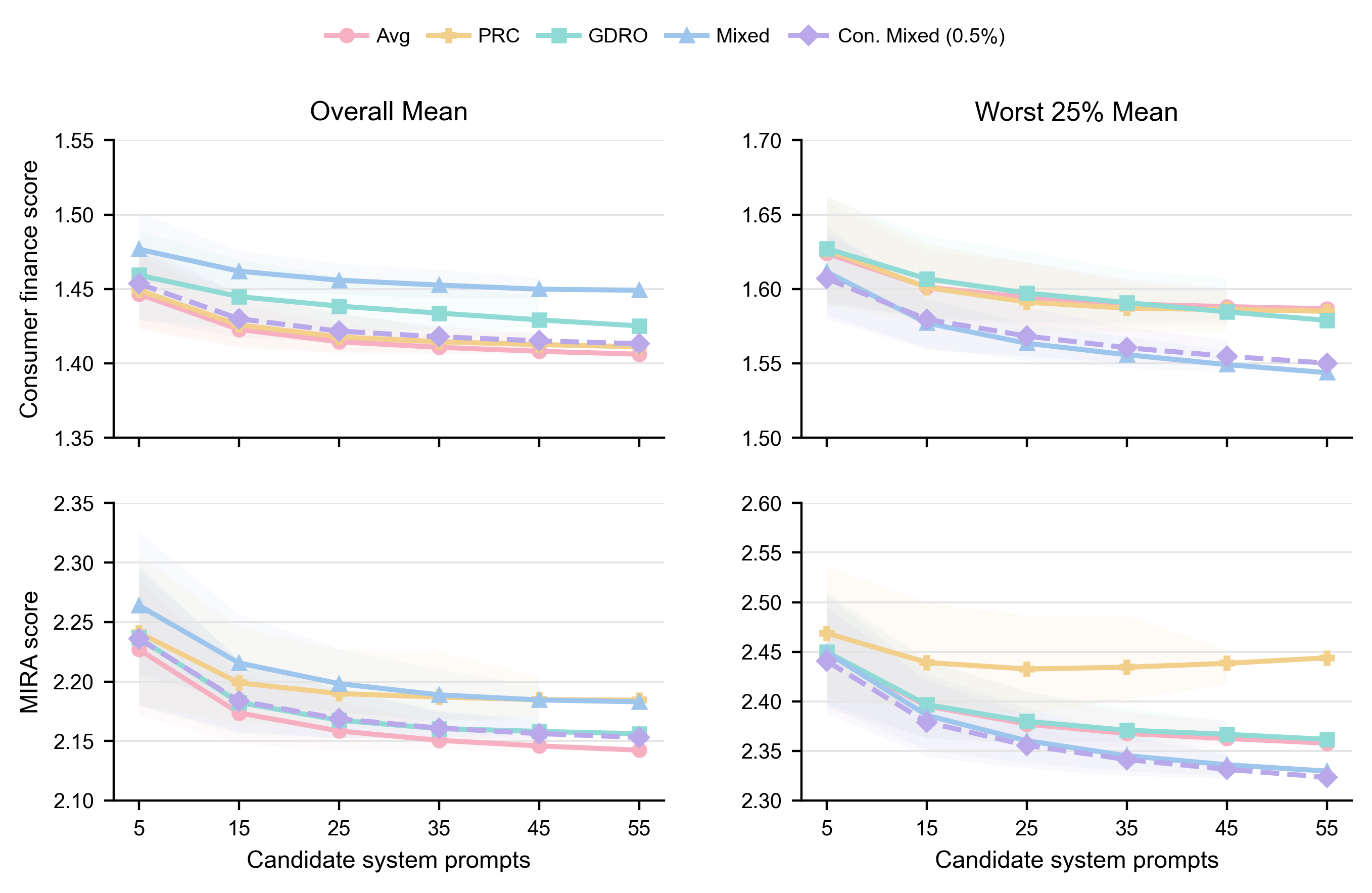}
    \caption{Development candidate-growth results for consumer
    finance and MIRA, averaged over five models. Lower is better.}
    \label{fig:candidate-growth-dev}
\end{figure}

\begin{figure}[t]
    \centering
    \includegraphics[width=\linewidth]{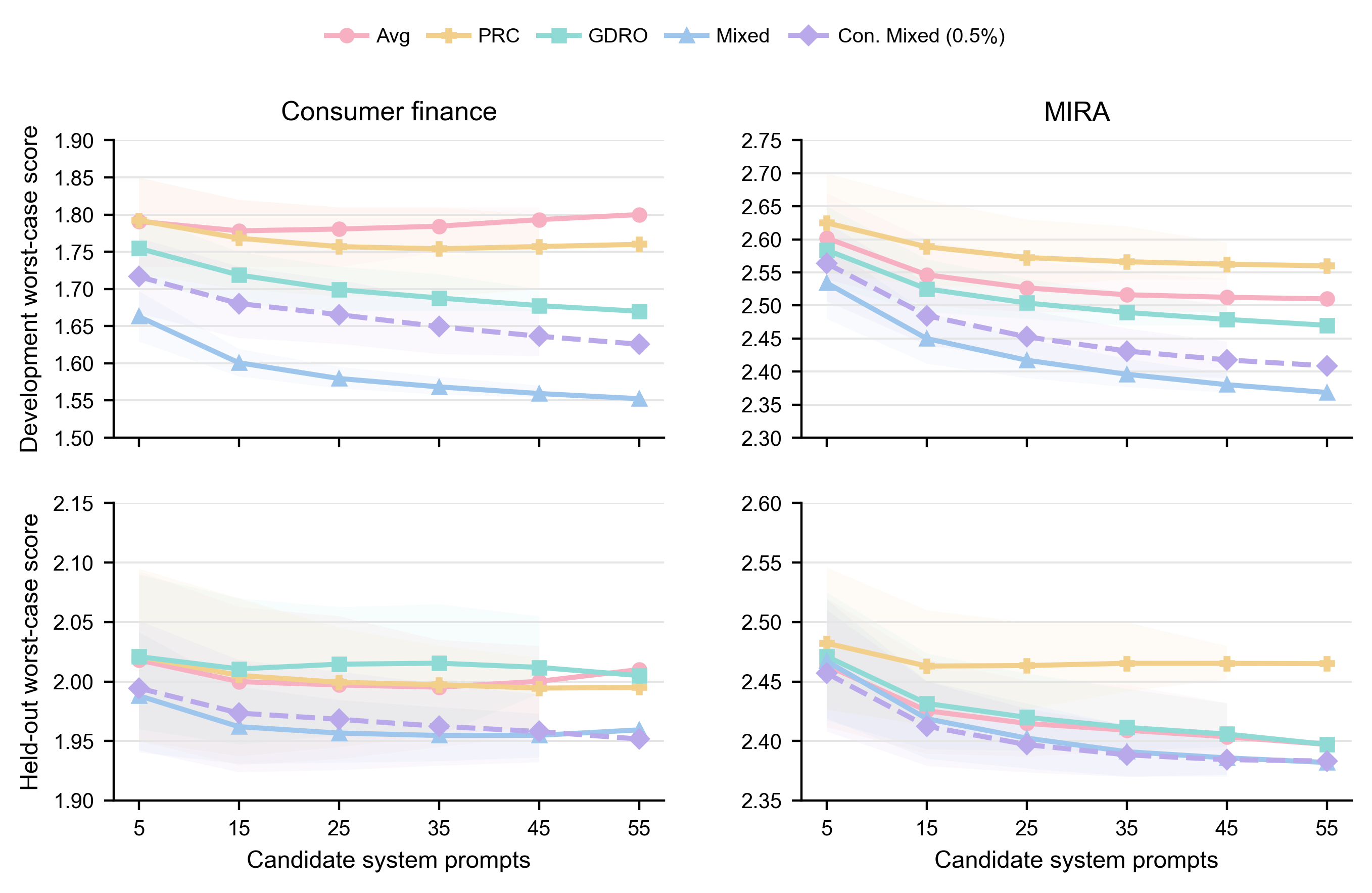}
    \caption{Worst-case candidate-growth results on development and
    held-out test data for consumer finance and MIRA, averaged over five models. Lower is better.}
    \label{fig:candidate-growth-worst}
\end{figure}

Table~\ref{tab:candidate-growth-recovery} further quantifies how much of the improvement from $K=5$ to $K=55$ is recovered by each smaller pool. Let $W_K^{d,s}$ denote the Constrained Mixed GroupDRO (0.5\%) joint worst-case score for pool size $K$, averaged over five models for domain $d$ and split $s$. We compute the recovered improvement as
\[
\operatorname{Recovery}(K) = 100\times \frac{W_5^{d,s}-W_K^{d,s}}{W_5^{d,s}-W_{55}^{d,s}}.
\]
Because lower scores are better, the numerator measures the worst-case reduction obtained at pool size $K$, while the denominator measures the total reduction obtained by increasing the pool from $K=5$ to $K=55$. By construction, recovery is 0\% at $K=5$ and 100\% at $K=55$. For example, the 56.24\% value for consumer-finance development data means that $K=25$ achieves 56.24\% of the total worst-case improvement obtained with the full pool. Across both domains, $K=25$ recovers at least 56.24\% of the full-pool improvement, and $K=35$ recovers at least 73.98\%. These results show that moderate-size pools already provide a substantial share of the benefit observed with 55 system prompts.

\begin{table}[t]
\centering
{\small
\setlength{\tabcolsep}{1mm}
\begin{tabular}{@{}llrrrr@{}}
\toprule
Domain & Split & $K=15$ & $K=25$ & $K=35$ & $K=45$ \\
\midrule
Finance & Dev  & 39.79 & 56.24 & 73.98 & 88.10 \\
Finance & Test & 49.07 & 61.21 & 74.92 & 85.25 \\
MIRA    & Dev  & 51.05 & 71.70 & 85.52 & 94.16 \\
MIRA    & Test & 60.23 & 81.20 & 92.96 & 98.36 \\
\bottomrule
\end{tabular}
}
\caption{Percentage of the $K=5$ to $K=55$ improvement in the Constrained Mixed (0.5\%) joint worst-case score recovered by smaller system-prompt pools, averaged over five models. By construction, $K=5$ is 0\% and $K=55$ is 100\%. Higher is better.}
\label{tab:candidate-growth-recovery}
\end{table}

\section{Family-Level Weight Analysis}
\label{app:family}
Figure~\ref{fig:family-weight} aggregates the Constrained Mixed GroupDRO weights by mitigation family. For each model and domain, the weight of a family is the sum of the weights assigned to all system prompts in that family. This shows whether the mixture combines different mitigation goals or different wording variants within the same family.

In MIRA, the largest weights generally fall on completeness-related families, although some models also use actionability and prompts that avoid referral-only answers. In consumer finance, the weights are more concentrated on actionability. For GLM5 and Llama in consumer finance, all family-level weight falls on the actionability family, even though the weight may be distributed across multiple prompts within that family. Thus, complementarity can arise both across mitigation families and across wording variants within the same family.

\begin{figure}[t]
    \centering
    \includegraphics[width=\linewidth]{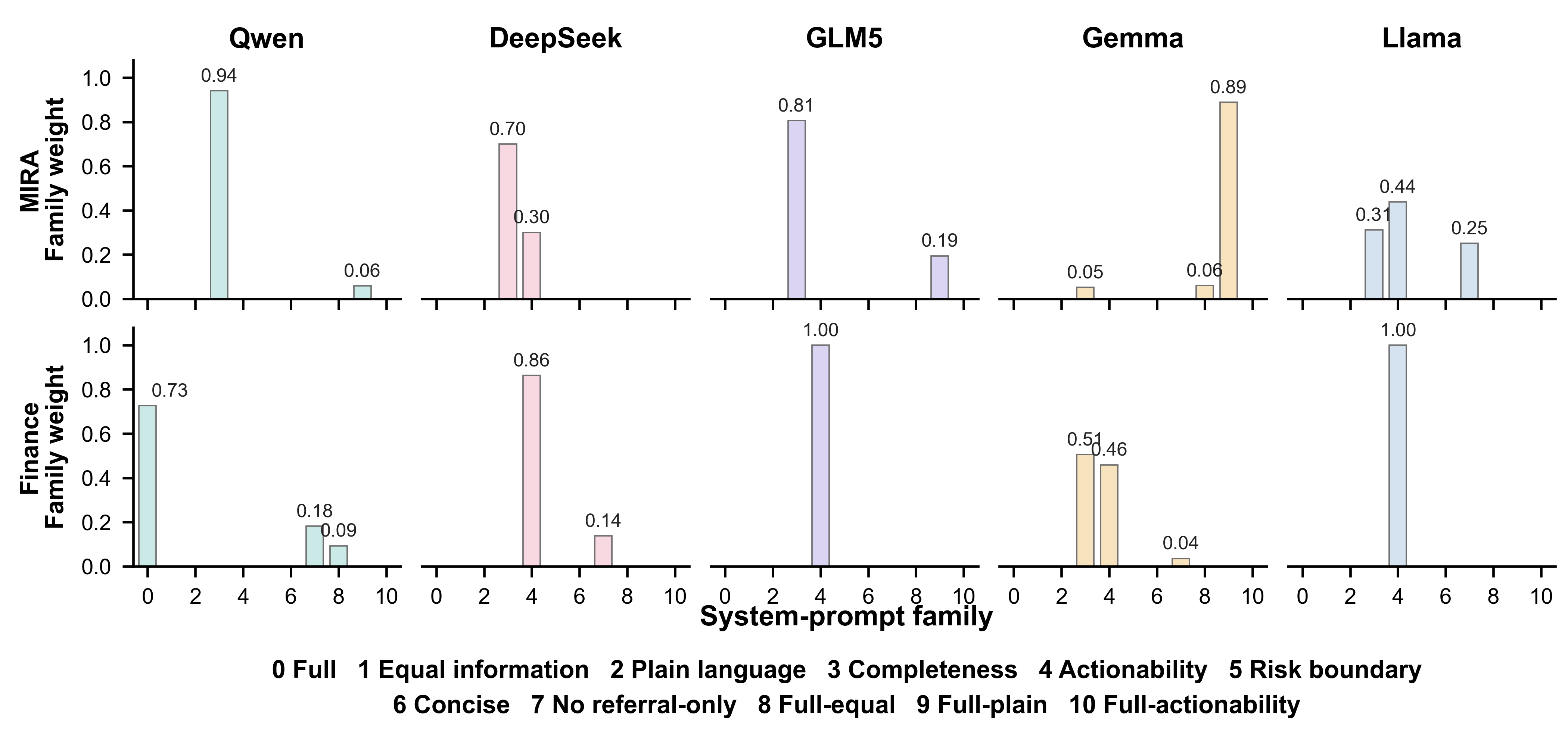}
    \caption{System-prompt weights aggregated by family.}
    \label{fig:family-weight}
\end{figure}

\section{Development Mean-Tail Pareto Tradeoff}
\label{app:pareto}
Figure~\ref{fig:dev-pareto} reports the development mean--tail tradeoff for both domains. Lower values on both axes are better. Constrained Mixed (0.5\%) provides a strong balance between the Overall Mean and Worst 25\% Mean. In both domains, no tested setting achieves lower development scores on both measures, supporting our use of the 0.5\% constraint in the main experiments.

\begin{figure}[t]
    \centering
    \includegraphics[width=\linewidth]{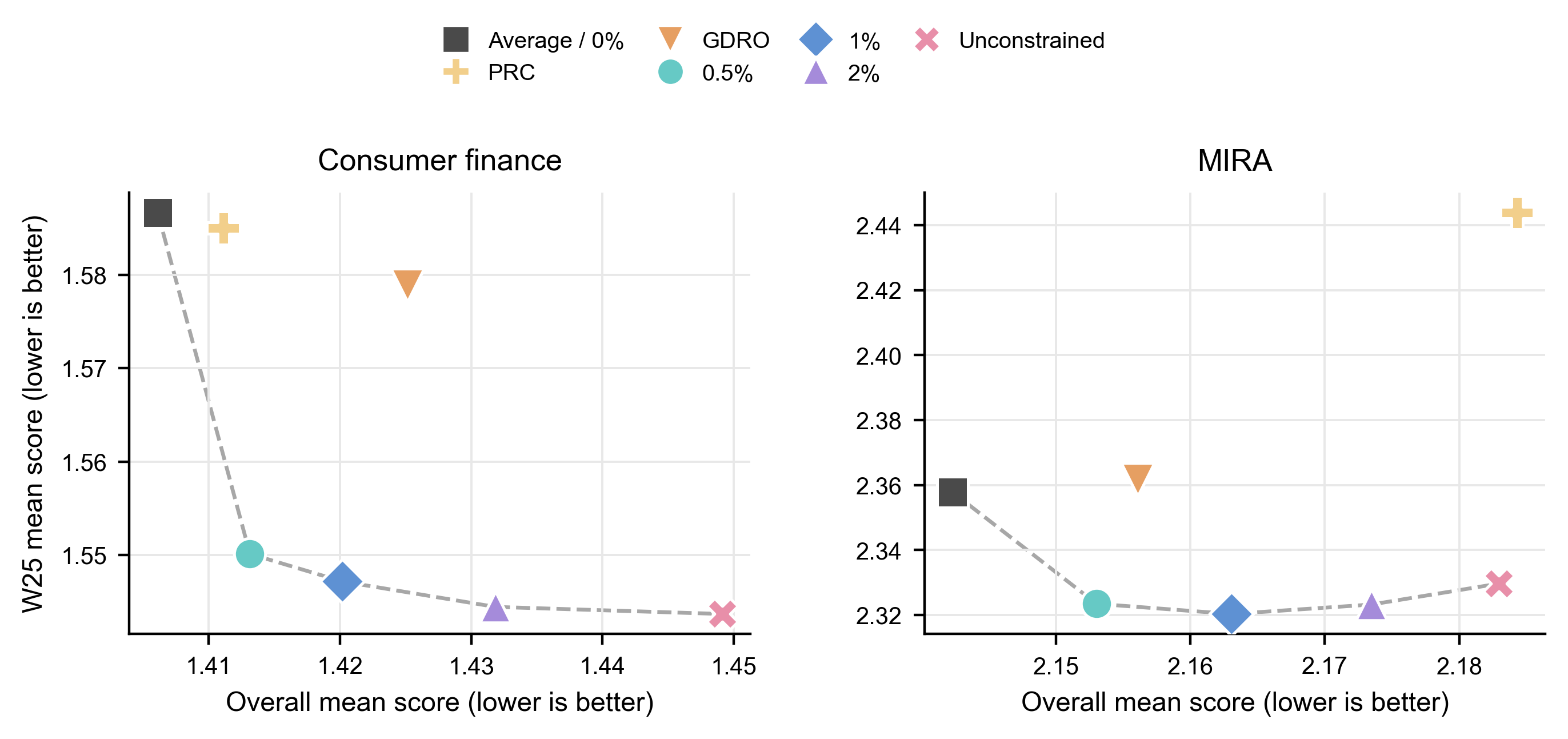}
    \caption{Development mean-tail tradeoff for consumer finance
    and MIRA, averaged over five models. Lower is better.}
    \label{fig:dev-pareto}
\end{figure}

\end{document}